\documentclass{article}
\usepackage{iclr2027_conference,times}

\usepackage[utf8]{inputenc}
\usepackage[T1]{fontenc}
\usepackage{url}
\usepackage{booktabs}
\usepackage{amsfonts}
\usepackage{amsmath}
\usepackage{amssymb}
\usepackage{amsthm}
\usepackage{dsfont}
\usepackage{nicefrac}
\usepackage{microtype}
\usepackage{graphicx}
\usepackage{bbm}
\usepackage{multirow}
\usepackage{siunitx}
\usepackage[version=4]{mhchem}
\usepackage{hyperref}
\usepackage{tabularx}
\usepackage[table]{xcolor}
\usepackage{wrapfig}
\usepackage{capt-of}
\usepackage{tikz}
\usepackage{float}
\usetikzlibrary{arrows.meta,calc,positioning}
\usepackage{pifont}
\newcommand{\cmark}{\ding{51}}
\newcommand{\xmark}{\ding{55}}
\newcommand{\ind}{\mathbbm{1}}
\theoremstyle{plain}
\newtheorem{proposition}{Proposition}
\definecolor{accent}{RGB}{30,70,120}

\newcommand\msmall[1]{\mbox{\scriptsize\ensuremath{#1}}}

\title{OxiGen: Oxidation-State-Aware Crystal \\ Generation}

\author{%
Dylan John$^{1}$, Kim E. Jelfs$^{1}$, Alex M. Ganose$^{1}$, Eleonora Giunchiglia$^{2}$ \\
$^{1}$Department of Chemistry, Imperial College London \\
$^{2}$Department of Electrical and Electronic Engineering, Imperial College London \\
\texttt{d.john25@imperial.ac.uk}
}

\iclrfinalcopy

\begin{document}
\maketitle

\begin{abstract}

Generative models have the potential to accelerate inorganic materials discovery by enabling inverse design, but generating experimentally realisable crystals remains challenging. Oxidation states are widely used to assess the compositional validity of crystals and guide inorganic materials discovery. While existing generative models for crystals can generate materials with charge-neutral oxidation-state assignments, they poorly reproduce the distributions of oxidation states observed in synthesised materials. To address this limitation, we propose OxiGen, an oxidation-state-aware crystal diffusion model that explicitly represents oxidation states during generation. OxiGen enforces global charge neutrality by construction using a structured output layer with exact inference over a finite-state automaton. Empirically, OxiGen substantially improves oxidation-state fidelity, generates the highest rate of stable, unique, and novel crystals among evaluated methods, and maintains high compositional validity even under property conditioning.

\end{abstract}

\section{Introduction}
\label{sec:intro}

The discovery of new inorganic materials is central to the development of technologies including solid-state batteries, photovoltaic cells and electrocatalysis \citep{yao_machine_2023}. Generative models have the potential to accelerate materials discovery by enabling \emph{inverse design}, in which candidate crystal structures with desired properties are generated directly rather than selected by screening existing databases \citep{park_has_2024}. Recent approaches evaluate the stability of generated crystals using their predicted energy above the convex hull \citep{zeni_generative_2025,miller_flowmm_2024,veljkovic_crystalite_2026}, but a low predicted energy above hull does not guarantee that a crystal is experimentally realisable \citep{park_closing_2026}. Chemical heuristics derived from theory and empirical observation therefore remain important for guiding materials discovery.

\emph{Oxidation states} are charges assigned to atoms through simple electron-counting rules \citep{karen_oxidation_2015}, and are a key heuristic for understanding the chemical bonding and properties of inorganic crystals \citep{walsh_oxidation_2018}. The oxidation states assigned to a crystal must satisfy \emph{charge neutrality}, meaning that they sum to zero. Each element can adopt a specific set of oxidation states, determined by its electronic configuration and orbital energies, and the observed frequencies of these oxidation states in synthesised materials vary widely. \emph{Compositional validity} evaluates whether the chemical composition of a crystal has at least one valid charge-neutral oxidation-state assignment \citep{davies_smact_2019}, but does not account for the observed frequencies of the required oxidation states. Rare oxidation states often occur only in particular bonding environments, under specific external conditions, or through specialised synthesis routes \citep{mastej_chemical_2026}. Consequently, crystals with uncommon oxidation-state assignments may be less likely to be experimentally realisable.

Recent generative models for crystals achieve high compositional validity, but many generated crystals can only be charge balanced using rare oxidation states \citep{mastej_chemical_2026}. Filtering based on predicted energy above hull does not necessarily address this issue: analysis of crystals generated by GNoME \citep{merchant_scaling_2023} identified structures predicted to lie on the hull that nevertheless exhibit improbable oxidation-state assignments \citep{cheetham_artificial_2024}. This highlights a limitation of current stability metrics, which may not capture whether the required oxidation states are chemically plausible. We therefore argue that effective generative models for crystals should not only generate stable and compositionally valid crystals, but also reproduce the distributions of oxidation states observed in synthesised materials.

In this work, we propose a novel approach to oxidation-state-aware crystal generation that explicitly represents oxidation states during generation, allowing the model to learn oxidation-state frequencies directly from the training data. However, explicitly generating oxidation states requires that the predicted oxidation states for all atoms jointly satisfy charge neutrality. To address this, we introduce a \emph{structured output layer} that exploits the simple additive structure of the charge-neutrality constraint for tractable structured inference. We then integrate this layer into a crystal diffusion model to enforce charge neutrality during sampling. Our main contributions are:
\begin{itemize}
\item We propose a structured output layer that uses dynamic programming over a finite-state automaton for exact charge-neutral inference.
\item We introduce OxiGen, an oxidation-state-aware crystal diffusion model that explicitly represents oxidation states and enforces charge neutrality during generation. 
\item We show empirically that OxiGen substantially improves oxidation-state fidelity over existing crystal generative models, generates the highest rate of stable, unique, and novel crystals among evaluated methods, and maintains high compositional validity even under property conditioning.
\end{itemize}

\section{Related Work}
\label{sec:related}

\noindent\textbf{Crystal Generative Models.}\quad Recent work has explored a wide range of generative methods for inorganic crystals, including diffusion models \citep{xie_crystal_2022, jiao_crystal_2023, zeni_generative_2025, joshi_all-atom_2025, veljkovic_crystalite_2026,  park_guiding_2026,yi_crystaldit_2026}, conditional flow matching \citep{miller_flowmm_2024, luo_crystalflow_2025}, and stochastic interpolants \citep{hollmer_open_2025}. Several approaches incorporate crystallographic constraints directly into the representation or generative process, particularly space-group symmetry \citep{jiao_space_2024, levy_symmcd_2024, chang_space_2025}. Other approaches address compositional validity. Crystal-GFN enforces charge neutrality as a hard constraint within a GFlowNet \citep{ai4science_crystal-gfn_2023}, but only generates composition, space group and lattice parameters, not atomic coordinates. \citet{mastej_chemical_2026} use an oxidation-state-based chemical validity operator as a reinforcement-learning reward, but do not include oxidation states within the generative model. CrysVCD proposes a chemical-formula transformer that represents compositions as sequences of element--oxidation-state pairs, but requires a separate diffusion model to generate the corresponding crystal structure \citep{cheng_enhancing_2026}. ChargeDIFF instead jointly generates the crystal structure and its continuous charge density, but focuses on modelling electronic structure rather than addressing compositional validity \citep{park_generative_2026}.  In contrast to these approaches, OxiGen incorporates oxidation states directly within a crystal diffusion model, jointly generating composition and structure while enforcing charge neutrality by construction.

\noindent\textbf{Structured Prediction.}\quad Structured prediction layers are a common approach for enforcing hard constraints on discrete neural network predictions \citep{dragone_neuro-symbolic_2021,hoernle_multiplexnet_2022,ahmed_semantic_2022}. In particular, Semantic Probabilistic Layers \citep{ahmed_semantic_2022} construct a structured distribution by restricting the support of a factorised predictive distribution to assignments satisfying a symbolic constraint. More recently, constraints have been incorporated into the sampling process of discrete diffusion models \citep{cardei_constrained_2025,christopher_neuro-symbolic_2025,suresh_dingo_2025,tomasi_primal-dual_2026}. DINGO \citep{suresh_dingo_2025} represents regular-language constraints with finite automata and uses dynamic programming for exact constrained maximum a posteriori (MAP) decoding in diffusion language models. In this work, we represent charge neutrality as a finite-state automaton and show that the resulting dynamic program supports both exact MAP decoding and exact constrained sampling over oxidation-state assignments. Concurrent work by \citet{dang_constrained_2026} independently develops exact constrained sampling over finite automata for diffusion language models.

\section{Preliminaries}
\label{sec:prelim}

\noindent\textbf{Crystal representation.}\quad A crystal is a periodic arrangement of atoms, represented by a repeating unit cell whose translations tile 3D space. We refer to each indexed atom in the unit cell as an \emph{atomic site}. A unit cell containing $N$ atomic sites can be represented by the triple $\boldsymbol{M}=(\boldsymbol{A},\boldsymbol{F},\boldsymbol{L})$. Here, $\boldsymbol{A}=(a_1,a_2,\ldots,a_N)\in\mathcal{A}^N$ gives the element identity of each site, where $\mathcal{A}$ is the set of chemical elements; $\boldsymbol{F}=[\boldsymbol{f}_1,\boldsymbol{f}_2,\ldots,\boldsymbol{f}_N]\in[0,1)^{3\times N}$ gives the fractional coordinates of each site within the unit cell; and $\boldsymbol{L}=[\boldsymbol{\ell}_1,\boldsymbol{\ell}_2,\boldsymbol{\ell}_3]\in\mathbb{R}^{3\times 3}$ is called the \emph{lattice matrix}, whose columns are the three basis vectors defining the unit cell. The corresponding Cartesian coordinates are $\boldsymbol{X}=\boldsymbol{L}\boldsymbol{F}$.

\noindent\textbf{Oxidation states.}\quad Formal oxidation states are integer-valued charges associated with the atoms in a crystal. For each element $a\in\mathcal{A}$, we denote its set of admissible oxidation states by $\mathcal{Z}(a)\subset\mathbb{Z}$. We define a \emph{species} as an element--oxidation-state pair $s=(a,z)$, and let $\mathcal{S}=\{(a,z):a\in\mathcal{A},z\in\mathcal{Z}(a)\}$ denote the species vocabulary. Charge neutrality requires that the oxidation states of the atoms in a crystal sum to zero, $\sum_{i=1}^{N} z_i=0$. While the majority of synthesised inorganic crystals have a charge-neutral oxidation-state assignment, formal integer oxidation states can become ambiguous in certain classes of materials. In this work we assume all crystals have a charge-neutral integer oxidation-state assignment. We discuss the limitations of this assumption in Section~\ref{sec:conclusion}.

\noindent\textbf{Discrete diffusion.}\quad Discrete diffusion models generate categorical data by corrupting clean samples through a fixed forward process $q$ and learning a reverse process that progressively denoises them. Let $\boldsymbol{S}_0=(s_{0,1},\ldots,s_{0,N})\in\mathcal{S}^N$ denote a clean \emph{species assignment} over the finite vocabulary $\mathcal{S}$. Absorbing-state diffusion independently replaces each species with a mask token $m$ according to a schedule $\alpha_t\in[0,1]$ \citep{austin_structured_2021,sahoo_simple_2024}:
\begin{equation}
q(s_{t,i}\mid s_{0,i})
=
\begin{cases}
1-\alpha_t, & s_{t,i}=m,\\
\alpha_t, & s_{t,i}=s_{0,i},
\end{cases}
\qquad
q(\boldsymbol{S}_t\mid \boldsymbol{S}_0)
=
\prod_{i=1}^N
q(s_{t,i}\mid s_{0,i}).
\label{eq:masking}
\end{equation}

Under the $x_0$-parameterisation, a denoiser predicts a factorised distribution over the clean state,
\begin{equation}
p_\theta(\boldsymbol{S}_0\mid\boldsymbol{S}_t,t)
=
\prod_{i=1}^{N}
p_{\theta,i}(s_{0,i}\mid\boldsymbol{S}_t,t).
\label{eq:clean-pred}
\end{equation}
The clean-state prediction can be combined with the tractable posterior of the forward process to obtain a reverse transition,
\begin{equation}
p_\theta(\boldsymbol{S}_r\mid\boldsymbol{S}_t,t)
=
\prod_{i=1}^{N}
\sum_{s\in\mathcal{S}}
q(s_{r,i}\mid s_{t,i},s_{0,i}=s)\,
p_{\theta,i}(s\mid\boldsymbol{S}_t,t),
\qquad r<t.
\label{eq:reverse}
\end{equation}

\section{Proposed Method: OxiGen}
\label{sec:method}

\subsection{Species representation}
\label{sec:species}

Previous generative models for crystals typically represent a crystal as $\boldsymbol{M}=(\boldsymbol{A},\boldsymbol{F},\boldsymbol{L})$, where $\boldsymbol{A}$ specifies the element identity of each atomic site. In these models, oxidation states are implicit, and many generated crystals can only be charge balanced using uncommon oxidation states \citep{mastej_chemical_2026}. We instead characterise each site by a species $s=(a,z)$ that jointly specifies its element identity and oxidation state. OxiGen therefore generates a categorical species token at each site, drawn from a vocabulary $\mathcal{S}$ of species observed in synthesised materials. The element assignment $\boldsymbol{A}$ is replaced by a species assignment $\boldsymbol{S}=(s_{1},\ldots,s_{N})\in\mathcal{S}^N$, where each $s_i\in\mathcal{S}$ has element $a(s_i)$ and oxidation state $z(s_i)$, and the overall crystal is represented as $\boldsymbol{M}=(\boldsymbol{S},\boldsymbol{F},\boldsymbol{L})$. For the species assignment to be chemically meaningful, $\boldsymbol{S}$ must satisfy charge neutrality, requiring the oxidation states across all sites to sum to zero.

\subsection{Diffusion Model}
\label{sec:diffusion}

OxiGen is a crystal diffusion model that jointly generates the species assignment, fractional coordinates, and lattice matrix. Starting from a clean crystal $\boldsymbol{M}_0=(\boldsymbol{S}_0,\boldsymbol{F}_0,\boldsymbol{L}_0)$, the forward diffusion process independently corrupts each component to obtain a noisy crystal $\boldsymbol{M}_t=(\boldsymbol{S}_t,\boldsymbol{F}_t,\boldsymbol{L}_t)$,
\begin{equation}
q(\boldsymbol{M}_t\mid\boldsymbol{M}_r)=q(\boldsymbol{S}_t\mid\boldsymbol{S}_r)\,q(\boldsymbol{F}_t\mid\boldsymbol{F}_r)\,q(\boldsymbol{L}_t\mid\boldsymbol{L}_r)\qquad 0\leq r<t\leq 1.
\label{eq:joint-forward}
\end{equation}
The reverse process is parameterised by a shared SE(3)-equivariant graph neural network that takes the complete noisy crystal as input. Here, we adapt MatterGen \citep{zeni_generative_2025}, retaining the network architecture and continuous diffusion processes for the lattice matrix and fractional coordinates, and modify only the discrete representation and diffusion process over species. We focus on these modifications here; further architectural and training details are provided in Appendix~\ref{app:oxigen}.

For the species component, we use an absorbing-state discrete diffusion process based on \citet{sahoo_simple_2024}. Given a noisy crystal state $\boldsymbol{M}_t$, the denoiser defines a factorised conditional distribution over the clean species assignment,
\begin{equation}
    p_\theta(\boldsymbol{S}_0\mid\boldsymbol{M}_t,t)
    = \prod_{i=1}^N
    p_{\theta,i}(s_{0,i}\mid\boldsymbol{M}_t,t).
    \label{eq:species-pred}
\end{equation}
Although each factor $p_{\theta,i}$ may depend on the complete noisy crystal $\boldsymbol{M}_t$, independently sampling from the site-wise distributions can in general produce assignments that violate charge neutrality. Enforcing charge neutrality as a global constraint therefore requires coupling the site-wise predictions during inference.

\subsection{Structured Inference}
\label{sec:inference}

\begin{figure*}[t]
    \centering
    \resizebox{\linewidth}{!}{%
        \input{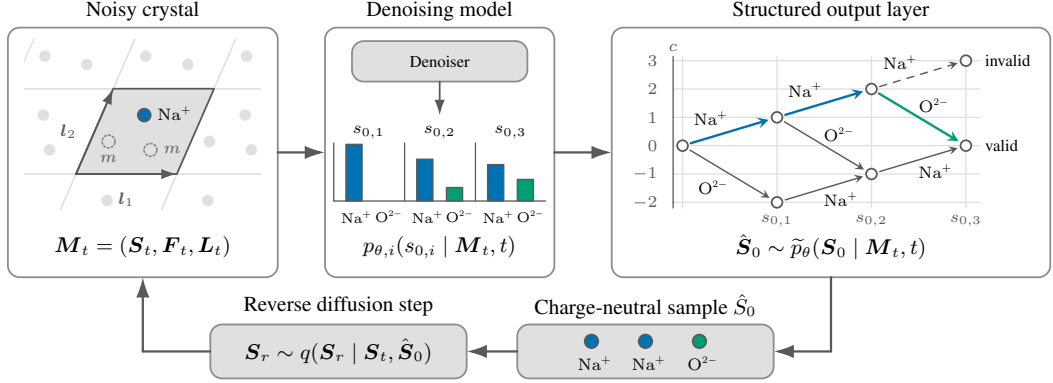}
    }

    \caption{
        Example of OxiGen discrete diffusion sampling over a species vocabulary of $\{\mathrm{Na}^+,\mathrm{O}^{2-}\}$. The noisy crystal state $\boldsymbol{M}_t$ is processed by the denoiser, which predicts site-wise species probabilities $p_{\theta,i}(s_{0,i}\mid\boldsymbol{M}_t,t)$. These are passed through the structured output layer to sample a charge-neutral assignment $\hat{\boldsymbol{S}}_0$, which is used in the reverse diffusion step.
    }
    \label{fig:method}
    \vspace{-0.5\baselineskip}
\end{figure*}

We define a structured joint distribution by conditioning the factorised distribution on charge neutrality,
\begin{equation}
\widetilde p_{\theta}(\boldsymbol{S}_0\mid \boldsymbol{M}_t,t)
=
\frac{1}{Z_{\theta}(\boldsymbol{M}_t,t)}
\ind\!\left[\sum_{i=1}^N z(s_{0,i})=0\right]
\prod_{i=1}^N
p_{\theta,i}(s_{0,i}\mid\boldsymbol{M}_t,t),
\label{eq:structured}
\end{equation}
where $Z_{\theta}(\boldsymbol{M}_t,t)=\sum_{\boldsymbol{S}_0\in\mathcal{S}^N}\ind\!\left[\sum_{i=1}^N z(s_{0,i})=0\right]\prod_{i=1}^Np_{\theta,i}(s_{0,i}\mid\boldsymbol{M}_t,t)$ is the normalising constant. The support of $ \widetilde p_{\theta}$ is restricted to charge-neutral assignments, so the constraint is satisfied by construction. However, computing $Z_\theta$ naively requires summing over an exponentially large space of species assignments. We instead represent the charge-neutrality constraint as a finite-state automaton, enabling exact inference by dynamic programming.

\noindent\textbf{Charge neutrality as a finite-state automaton.}\quad A \emph{finite-state automaton} consists of a finite set of states and a transition rule between them, with valid sequences corresponding to paths from an initial state to an accepting state. Charge neutrality has a simple additive structure that can be represented by a finite-state automaton. After assigning the first $i$ sites, whether the remaining sites can satisfy charge neutrality depends only on the cumulative charge
\begin{equation}
    c_i=\sum_{j=1}^i z(s_{0,j}).
\label{eq:charge}
\end{equation}
Charge-neutral assignments therefore correspond to paths through a layered finite-state automaton with initial state $c_0=0$, transition rule $c_i=c_{i-1}+z(s_{0,i})$, and accepting state $c_N=0$, as illustrated by the example in Figure~\ref{fig:method}. Dynamic programming over this automaton enables exact, tractable inference without summing over complete assignments. We adapt the standard forward--backward procedure for hidden Markov models \citep{rabiner_tutorial_1989}. Let $\mathcal{C}_i$ denote the set of reachable cumulative charges after assigning the first $i$ sites. The forward weight $F_i(c)$ gives the total probability mass of assignments to the first $i$ sites with cumulative charge $c$,
\begin{equation}
F_i(c)
=
\sum_{\substack{s_{0,1:i}\in\mathcal{S}^i\\
\sum_{j=1}^i z(s_{0,j})=c}}
\prod_{j=1}^i
p_{\theta,j}(s_{0,j}\mid\boldsymbol{M}_t,t),
\label{eq:fwd-weight}
\end{equation}
while the backward weight $B_i(c)$ gives the probability mass of assignments to the remaining sites that complete charge $c$ to zero,
\begin{equation}
B_i(c)
=
\sum_{\substack{s_{0,i+1:N}\in\mathcal{S}^{N-i}\\
c+\sum_{j=i+1}^N z(s_{0,j})=0}}
\prod_{j=i+1}^N
p_{\theta,j}(s_{0,j}\mid\boldsymbol{M}_t,t).
\label{eq:bwd-weight}
\end{equation}
These quantities satisfy the recursions
\begin{align}
F_0(c) &= \ind[c=0], &
F_i(c)
&=
\sum_{s\in\mathcal{S}}
p_{\theta,i}(s\mid\boldsymbol{M}_t,t)
F_{i-1}\!\left(c-z(s)\right),
\label{eq:fwd-rec}
\\
B_N(c) &= \ind[c=0], &
B_{i-1}(c)
&=
\sum_{s\in\mathcal{S}}
p_{\theta,i}(s\mid\boldsymbol{M}_t,t)
B_i\!\left(c+z(s)\right).
\label{eq:bwd-rec}
\end{align}

Although the recursions are sequential, inference is invariant to the ordering of the atomic sites because charge neutrality depends only on the sum of their oxidation states, which is commutative. After all $N$ sites have been assigned, $F_N(0)$ is by definition the total probability mass of charge-neutral assignments, which is the normalising constant $Z_{\theta}(\boldsymbol{M}_t,t)$. Similarly, $B_0(0)$ sums the same mass from the backward direction, so
\begin{equation}
Z_{\theta}(\boldsymbol{M}_t,t)=F_N(0)=B_0(0).
\label{eq:normaliser}
\end{equation}

To obtain the marginal probability that site $i$ is assigned species $s$, we sum the probability mass of all charge-neutral assignments consistent with $s_{0,i}=s$. For any cumulative charge $c$ before site $i$, $F_{i-1}(c)$ gives the mass of compatible prefix assignments, $p_{\theta,i}(s\mid\boldsymbol{M}_t,t)$ gives the site-wise probability of assigning $s$, and $B_i(c+z(s))$ gives the mass of suffix assignments that return the resulting charge to zero. Summing over all reachable $c$ and normalising by $Z_\theta$ therefore gives
\begin{equation}
\widetilde p_{\theta,i}(s\mid\boldsymbol{M}_t,t)
=
\frac{
p_{\theta,i}(s\mid\boldsymbol{M}_t,t)
}{
Z_{\theta}(\boldsymbol{M}_t,t)
}
\sum_{c\in\mathcal{C}_{i-1}}
F_{i-1}(c)\,
B_i\!\left(c+z(s)\right).
\label{eq:marginal}
\end{equation}
A complete forward or backward pass has time complexity $O(N^2|\mathcal{S}|(z_{\max}-z_{\min}))$, where $z_{\max}$ and $z_{\min}$ are the maximum and minimum oxidation states in the species vocabulary. We report empirical computational overhead in Appendix~\ref{app:oxigen-compute}.

\noindent\textbf{Charge-neutral MAP decoding.}\quad We can compute the MAP assignment under the charge-neutral structured distribution in Eq.~\eqref{eq:structured} using the Viterbi algorithm \citep{forney_viterbi_1973}. Let $V_i(c)$ denote the maximum probability of any species assignment to the first $i$ sites with cumulative charge $c$. The recursion is
\begin{equation}
V_i(c)
=
\max_{s\in\mathcal{S}}
p_{\theta,i}(s\mid\boldsymbol{M}_t,t)\,
V_{i-1}\!\left(c-z(s)\right),
\qquad
V_0(c)=\ind[c=0].
\label{eq:viterbi}
\end{equation}
For each $(i,c)$, we store the maximising species as a backpointer.

\begin{proposition}[Exact charge-neutral MAP decoding]
\label{prop:map}
Backtracking from $V_N(0)$ returns the exact MAP species assignment under $\widetilde p_{\theta}(\boldsymbol{S}_0\mid\boldsymbol{M}_t,t)$.
\end{proposition}

\noindent\textbf{Charge-neutral sampling.}\quad While the Viterbi recursion provides exact MAP decoding, the reverse diffusion process requires samples from the full structured distribution. During finite-step generation, a reverse step may unmask multiple sites simultaneously. Sampling these sites independently would discard the dependencies induced by charge neutrality. We therefore sample a joint clean species assignment from the structured distribution. Given cumulative charge $c_{i-1}$ after the first $i-1$ sites, the conditional distribution for site $i$ is
\begin{equation}
\widetilde{p}_\theta\!\left(
s_{0,i}=s
\mid c_{i-1},\boldsymbol{M}_t,t
\right)
=
\frac{
p_{\theta,i}(s\mid\boldsymbol{M}_t,t)\,
B_i\!\left(c_{i-1}+z(s)\right)
}{
B_{i-1}(c_{i-1})
},
\label{eq:sampler}
\end{equation}
where the backward weights $B_i(c)$ are computed using the backward recursion in Eq.~\eqref{eq:bwd-rec}. 

\begin{proposition}[Exact charge-neutral sampling]
\label{prop:sampling}
Sequentially sampling $s_{0,i}$ according to Eq.~\eqref{eq:sampler} for $i=1,\ldots,N$ produces an exact sample from
$\widetilde p_{\theta}(\boldsymbol{S}_0\mid\boldsymbol{M}_t,t)$.
\end{proposition}

At each OxiGen reverse step, we first sample a complete charge-neutral clean species assignment $\hat{\boldsymbol{S}}_0\sim\widetilde p_{\theta}(\boldsymbol{S}_0\mid\boldsymbol{M}_t,t)$ and then sample $\boldsymbol{S}_r\sim q(\boldsymbol{S}_r\mid\boldsymbol{S}_t,\hat{\boldsymbol{S}}_0)$ from the tractable posterior of the forward absorbing process. This procedure avoids explicitly marginalising over clean assignments as in Eq.~\eqref{eq:reverse}, but is equivalent to sampling from the corresponding marginal reverse transition under the structured distribution. It therefore preserves the dependencies among sites unmasked within the same reverse step and guarantees that the final species assignment is charge neutral. For OxiGen, the structured output layer is applied only during sampling, while training uses the standard factorised masked-diffusion objective. We derive the corresponding structured training objective in Appendix~\ref{app:diffusion} and evaluate it in Section~\ref{sec:de-novo}.

\section{Experiments}
\label{sec:experiments}

\subsection{Oxidation-State Prediction}
\label{sec:os-pred}

Training OxiGen requires crystal structures labelled with oxidation states, which are not provided for all structures in public crystal-structure databases such as the Materials Project \citep{jain_commentary_2013}. We therefore train a crystal transformer adapted from \citet{veljkovic_crystalite_2026} for oxidation-state prediction. To guarantee charge-neutral predictions, we use the structured Viterbi decoding introduced in Section~\ref{sec:inference}. We additionally modify the dynamic program so that mixed valence, where different sites of the same element have different oxidation states, is only permitted for elements with established mixed-valence chemistry. Further details are provided in Appendix~\ref{app:grouped-dp}. We compare against \texttt{BVAnalyzer} \citep{ong_python_2013}, a standard method based on empirical bond-valence analysis that assigns oxidation states subject to charge neutrality, and BERTOS \citep{fu_composition_2023}, a composition-only transformer that does not guarantee charge-neutral predictions.

\noindent\textbf{Evaluation.}\quad We evaluate oxidation-state prediction on 28,492 experimental crystal structures from the Inorganic Crystal Structure Database (ICSD) \citep{zagorac_recent_2019}, a curated, licensed database that provides oxidation-state labels for each atomic site. We use a composition-disjoint 22,879/2,831/2,782 train/validation/test split. We reimplement BERTOS and train it on the same split. For our model, we evaluate both standard argmax decoding and structured Viterbi decoding. Following \citet{fu_composition_2023}, we report the charge-neutrality rate, defined as the fraction of crystals with charge-neutral predictions; site accuracy, the fraction of atomic sites assigned the correct oxidation state; and crystal accuracy, the fraction of crystals for which all atomic sites are assigned correctly. \texttt{BVAnalyzer} fails for 9.1\% of test structures for which no assignment is compatible with the bond-valence model; we count these cases as incorrect for site and crystal accuracy and exclude them when computing charge neutrality. Further details on dataset curation, model architecture, and training are provided in Appendix~\ref{app:os-pred}. Table~\ref{tab:os-pred} shows that our model achieves higher site and crystal accuracy than both baselines. Charge-neutral Viterbi decoding enforces charge neutrality for all predictions and further improves site and crystal accuracy over argmax decoding.

\noindent\textbf{MP-20-OS dataset.}\quad We use our predictor to label MP-20, a subset of the Materials Project \citep{jain_commentary_2013} containing crystal structures with up to 20 atoms in the unit cell. We use the splits provided by \citet{zeni_generative_2025} and label each split independently. Following convention, structures composed exclusively of metals are assigned oxidation state zero at every site, while all remaining structures are labelled using our predictor. We remove 5.9\% of structures using the filtering procedure described in Appendix~\ref{app:mp20-labelling}, including 3.7\% with no charge-neutral oxidation-state assignment. The resulting MP-20-OS dataset contains 25,536/8,495/8,545 train/validation/test structures.

\begin{table}[t]
    \vspace{-1\baselineskip}
    \centering
    \small
    \setlength{\tabcolsep}{3.2pt}
    \renewcommand{\arraystretch}{1.08}

    \caption{Oxidation-state prediction results, reported as mean $\pm$ std over five independently seeded training runs; the deterministic BVAnalyzer is reported as a single value. Best values are shown in \textbf{bold} and second-best values are \underline{underlined}.}
    \label{tab:os-pred}
    \vspace{0.5\baselineskip}

    \begin{tabular}{@{}lcccc@{}}
        \toprule
        Method
        & Decoding
        & Charge Neutrality (\%) $\uparrow$
        & Site Accuracy (\%) $\uparrow$
        & Crystal Accuracy (\%) $\uparrow$ \\
        \midrule

        \texttt{BVAnalyzer}
            & --
            & $\mathbf{100.00}$
            & 88.19
            & 86.02 \\

        BERTOS
            & Argmax
            & $92.83\msmall{\pm 0.22}$
            & $97.81\msmall{\pm 0.17}$
            & $90.93\msmall{\pm 0.24}$ \\

        \midrule

        \multirow{2}{*}{\textbf{Ours}}
            & Argmax
            & $\underline{94.37}\msmall{\pm 0.26}$
            & $\underline{98.47}\msmall{\pm 0.10}$
            & $\underline{92.79}\msmall{\pm 0.18}$ \\

            & Viterbi
            & $\mathbf{100.00}\msmall{\pm 0.00}$
            & $\mathbf{98.74}\msmall{\pm 0.09}$
            & $\mathbf{96.00}\msmall{\pm 0.25}$ \\

        \bottomrule
    \end{tabular}

    \vspace{-0.5\baselineskip}
\end{table}

\subsection{Crystal Generation}
\label{sec:de-novo}

We evaluate OxiGen on \textit{de novo} crystal generation using MP-20-OS and compare against five strong baseline methods: DiffCSP \citep{jiao_crystal_2023}, MatterGen \citep{zeni_generative_2025}, FlowMM \citep{miller_flowmm_2024}, OMatG \citep{hollmer_open_2025}, and Crystalite \citep{veljkovic_crystalite_2026}. We retrain all methods on MP-20-OS using the authors’ original code and reported hyperparameters over five independent random seeds, and sample 10,240 crystals from each method for each seed.

\begin{table*}[t]
    \centering
    \small
    \setlength{\tabcolsep}{4.2pt}
    \renewcommand{\arraystretch}{1.10}
    \caption{\textit{De novo} generation results, reported as mean $\pm$ std over five independently seeded training runs, with 10,240 generated structures per seed. Best values are shown in \textbf{bold} and second-best values are \underline{underlined}. MP-20-OS denotes the test set.}
    \vspace{0.5\baselineskip}
    \label{tab:de-novo}

    \begin{tabular}{@{}lcccccc@{}}
        \toprule
        \multirow{3}{*}{Method}
        & \multicolumn{3}{c}{\textbf{Oxidation-State Fidelity}}
        & \multicolumn{3}{c}{\textbf{Stability \& Relaxation}} \\
        \cmidrule(lr){2-4}
        \cmidrule(lr){5-7}

        & Comp. Val.
        & $f_{\mathrm{OS}}$
        & $d_{\mathrm{OS}}$
        & S.U.N.
        & $E_{\mathrm{hull}}$
        & RMSD \\

        & (\%) $\uparrow$
        & (\%)
        & $\downarrow$
        & (\%) $\uparrow$
        & (eV/atom) $\downarrow$
        & (\AA) $\downarrow$ \\
        \midrule

        MP-20-OS
            & 91.73
            & 53.78
            & --
            & --
            & --
            & -- \\

        \midrule

        DiffCSP
            & $78.47\msmall{\pm 2.55}$
            & $35.67\msmall{\pm 1.05}$
            & $0.1566\msmall{\pm 0.0037}$
            & $18.24\msmall{\pm 1.38}$
            & $0.2007\msmall{\pm 0.0101}$
            & $0.3555\msmall{\pm 0.0427}$ \\

        MatterGen
            & $81.18\msmall{\pm 0.94}$
            & $37.89\msmall{\pm 1.02}$
            & $0.1537\msmall{\pm 0.0039}$
            & $\underline{24.54}\msmall{\pm 0.91}$
            & $0.1799\msmall{\pm 0.0085}$
            & $\underline{0.1747}\msmall{\pm 0.0203}$ \\

        FlowMM
            & $79.87\msmall{\pm 0.20}$
            & $37.34\msmall{\pm 0.65}$
            & $0.1626\msmall{\pm 0.0037}$
            & $13.85\msmall{\pm 0.39}$
            & $0.2481\msmall{\pm 0.0104}$
            & $0.6704\msmall{\pm 0.0347}$ \\

        OMatG
            & $82.32\msmall{\pm 0.48}$
            & $\underline{43.08}\msmall{\pm 0.35}$
            & $\underline{0.1284}\msmall{\pm 0.0060}$
            & $19.03\msmall{\pm 0.78}$
            & $0.1831\msmall{\pm 0.0030}$
            & $0.5244\msmall{\pm 0.0278}$ \\

        Crystalite
            & $\underline{82.66}\msmall{\pm 0.51}$
            & $42.45\msmall{\pm 1.04}$
            & $0.1350\msmall{\pm 0.0017}$
            & $22.83\msmall{\pm 0.76}$
            & $\mathbf{0.1437}\msmall{\pm 0.0083}$
            & $0.2496\msmall{\pm 0.0198}$ \\

        \midrule

        \textbf{OxiGen}
            & $\mathbf{90.65}\msmall{\pm 0.94}$
            & $\mathbf{53.01}\msmall{\pm 0.96}$
            & $\mathbf{0.0911}\msmall{\pm 0.0028}$
            & $\mathbf{26.38}\msmall{\pm 1.01}$
            & $\underline{0.1575}\msmall{\pm 0.0080}$
            & $\mathbf{0.1679}\msmall{\pm 0.0117}$ \\

        \bottomrule
    \end{tabular}
    \vspace{-1\baselineskip}
\end{table*}

\noindent\textbf{Evaluation metrics.}\quad Compositional validity, computed using \texttt{SMACT} \citep{davies_smact_2019}, measures whether a composition has a valid charge-neutral oxidation-state assignment. Although every crystal in MP-20-OS is labelled with a charge-neutral assignment, the test-set compositional validity is not perfect because the \texttt{SMACT} evaluation additionally requires electronegativity balance and does not permit mixed-valence assignments. Since compositional validity does not account for how common the required oxidation states are, we report two additional metrics. The oxidation-state frequency score $f_{\mathrm{OS}}$ is the geometric mean of the ICSD frequencies of distinct species in an assignment \citep{mastej_chemical_2026}, and measures how commonly those oxidation states occur in synthesised materials. We report the mean across crystals, with the goal of matching the test-set value rather than simply favouring the most common oxidation states. The oxidation-state distribution distance $d_{\mathrm{OS}}$ measures the Jensen--Shannon distance between the oxidation-state distributions of the generated and test sets for each element, reported as the mean across elements. For a consistent comparison across methods, we enumerate all valid assignments for each crystal using \texttt{SMACT}, rather than using oxidation states generated by OxiGen or inferred by our predictor. For $f_{\mathrm{OS}}$ we take the maximum score across these assignments and for $d_{\mathrm{OS}}$ we construct the distributions by marginalising over assignments weighted by their ICSD occurrence frequencies. Metallic alloys and elemental structures are excluded from all three metrics because oxidation states are not meaningfully defined for these systems. Full definitions and evaluation details are provided in Appendix~\ref{app:os-metrics}. 

\noindent\textbf{Oxidation-state fidelity.}\quad OxiGen substantially improves oxidation-state fidelity over the baselines, closely matching the MP-20-OS test-set values for compositional validity and oxidation-state frequency score while achieving the lowest oxidation-state distribution distance (Table~\ref{tab:de-novo}). As shown in Figure~\ref{fig:os-fidelity}, the baselines rely more heavily on uncommon oxidation states, while OxiGen closely reproduces both the test-set compositional validity curve and oxidation-state distributions.  

\noindent\textbf{Stability, uniqueness, and novelty.}\quad Alongside oxidation-state fidelity, we evaluate the ability of each method to generate stable, unique, and novel crystals. Since \textit{de novo} generation typically involves a trade-off between stability and novelty \citep{park_guiding_2026}, we use the proportion of stable, unique, and novel (S.U.N.) structures as the primary evaluation metric. Following \citet{zeni_generative_2025}, generated structures are relaxed with the MatterSim machine-learning interatomic potential \citep{yang_mattersim_2024}, and the predicted energies are used to compute energy above the convex hull, $E_{\mathrm{hull}}$. A structure is considered stable if $E_{\mathrm{hull}} \leq 0.1$ eV/atom. We additionally report the mean $E_{\mathrm{hull}}$ and the root-mean-square deviation (RMSD) between the generated and relaxed structures. Further evaluation details are provided in Appendix~\ref{app:sun-eval}. OxiGen achieves the highest S.U.N. rate, lowest relaxation RMSD and second-lowest mean $E_{\mathrm{hull}}$ among the evaluated methods (Table~\ref{tab:de-novo}). Relative to MatterGen, the increase in the S.U.N. rate corresponds to a higher proportion of stable structures, accompanied by a decrease in uniqueness and novelty (Appendix Table~\ref{tab:app-de-novo}), consistent with OxiGen generating fewer crystals with chemically implausible compositions. We further evaluate OxiGen using LeMat-GenBench \citep{betala_lematgenbench_2026} and the proxy metrics proposed by \citet{xie_crystal_2022} (Appendix Tables~\ref{tab:app-proxy} and~\ref{tab:leaderboard}). Overall, OxiGen substantially improves oxidation-state fidelity while maintaining strong performance on standard stability metrics.

\begin{figure}
    \centering    \includegraphics{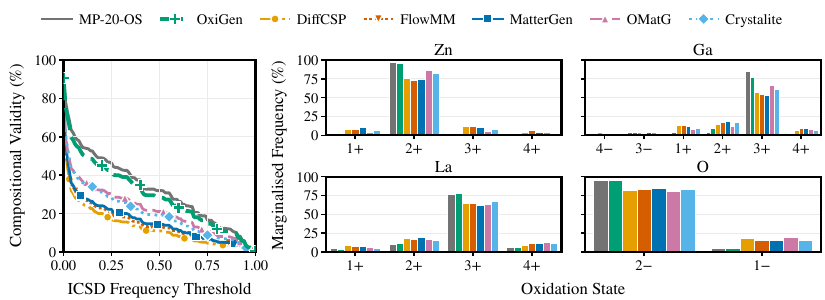}
    \vspace{-1\baselineskip}
    \caption{\textbf{Left}: Compositional validity as a function of the ICSD frequency threshold required for an oxidation state to be considered admissible. \textbf{Right}: Marginal oxidation-state distributions for selected elements. MP-20-OS denotes the test set in both panels.}
    \label{fig:os-fidelity}
    \vspace{-0.5\baselineskip}
\end{figure}

\begin{table}[t] 
\vspace{-1\baselineskip}
\centering 
\small 
\setlength{\tabcolsep}{4.5pt} 
\renewcommand{\arraystretch}{1.08} 
\caption{\textit{De novo} generation ablation results, reported as mean $\pm$ std over five independently seeded training runs, with 10,240 generated structures per seed. Best values are shown in \textbf{bold} and second-best values are \underline{underlined}. The shaded row is OxiGen.}
\vspace{0.5\baselineskip}
\label{tab:ablation}

\begin{tabular}{@{}lcccccc@{}}
    \toprule
    \multirow{2}{*}{Representation}
    & \multicolumn{2}{c}{Structured Output Layer}
    & Charge Neutrality
    & Comp. Val.
    & $f_{\mathrm{OS}}$
    & $d_{\mathrm{OS}}$ \\
    \cmidrule(lr){2-3}
    & Training
    & Sampling
    & (\%) $\uparrow$
    & (\%) $\uparrow$
    & (\%)
    & $\downarrow$ \\
    \midrule

    Element
        & \xmark
        & \xmark
        & --
        & $81.12\msmall{\pm 0.46}$
        & $37.78\msmall{\pm 1.16}$
        & $0.1589\msmall{\pm 0.0050}$ \\

    Species
        & \xmark
        & \xmark
        & $\underline{39.08}\msmall{\pm 2.60}$
        & $82.31\msmall{\pm 0.77}$
        & $40.54\msmall{\pm 1.29}$
        & $0.1368\msmall{\pm 0.0047}$ \\

    \rowcolor{gray!15}
    Species
        & \xmark
        & \cmark
        & $\mathbf{100.00}\msmall{\pm 0.00}$
        & $\underline{90.65}\msmall{\pm 0.94}$
        & $\mathbf{53.01}\msmall{\pm 0.96}$
        & $\underline{0.0911}\msmall{\pm 0.0028}$ \\

    Species
        & \cmark
        & \cmark
        & $\mathbf{100.00}\msmall{\pm 0.00}$
        & $\mathbf{90.70}\msmall{\pm 0.67}$
        & $\underline{52.65}\msmall{\pm 0.76}$
        & $\mathbf{0.0874}\msmall{\pm 0.0029}$ \\

    \bottomrule
\end{tabular}
\vspace{-1\baselineskip}
\end{table}

\noindent\textbf{Ablation study.}\quad
We ablate the species representation and structured output layer to isolate the contribution of each component to OxiGen (Table~\ref{tab:ablation}). All ablation models use the same masked discrete diffusion process as OxiGen, differing only in the discrete state representation and whether the structured output layer is applied during training and sampling. We report the three oxidation-state metrics and the fraction of generated species assignments that satisfy charge neutrality, with the corresponding S.U.N. metrics reported in Appendix Table~\ref{tab:app-ablation-stable}. The species representation alone provides only a modest improvement over the element representation, with the majority of generated assignments violating charge neutrality. Structured sampling enforces charge neutrality by construction and substantially improves all three oxidation-state metrics. These results suggest that explicitly representing oxidation states is insufficient on its own, and that enforcing global charge neutrality during sampling is needed to ensure chemically meaningful oxidation-state assignments. Training with the structured objective does not provide any consistent additional improvement while adding training overhead, and is therefore not used to train OxiGen.

\noindent\textbf{Conditional generation.}\quad
Inverse design requires generating crystals with desired properties while maintaining generation quality. This becomes more challenging for targets that are weakly represented in the training data, as conditioning can push generation away from high-density regions of the learned distribution. Following \citet{zeni_generative_2025}, we fine-tune MatterGen and OxiGen on band-gap labels from MP-20-OS and generate 1,024 crystals conditioned on target band gaps of 1, 3, 5, and 7~eV using classifier-free guidance. These targets span the MP-20-OS band-gap distribution, which contains progressively fewer structures at higher band gaps (Figure~\ref{fig:band-gap}). Band gaps are predicted from the relaxed structures using iComFormer, an SE(3)-invariant crystal graph transformer for property prediction \citep{yan_complete_2024}. Further details on model fine-tuning and training the band-gap predictor are provided in Appendix~\ref{app:band-gap}. As shown in Figure~\ref{fig:band-gap}, both models shift the predicted band-gap distributions towards the conditioning targets. However, MatterGen's compositional validity decreases sharply as the target band gap increases, whereas OxiGen maintains consistently high validity across the full range. OxiGen also achieves a higher S.U.N. rate than MatterGen at all four conditioning targets (Appendix Table~\ref{tab:app-bandgap-conditional}). These results demonstrate the value of structured sampling for inverse design, where conditioning on sparsely represented targets can otherwise lead to chemically invalid compositions.

\begin{figure}[t]
    \centering
    \includegraphics{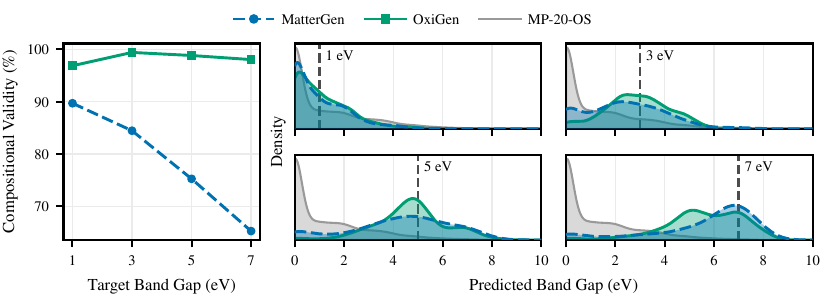}
    \vspace{-1.5\baselineskip}
    \caption{\textbf{Left}: Compositional validity at target band gaps of 1, 3, 5, and 7 eV. \textbf{Right}: Kernel density estimates of the predicted band-gap distributions for structures generated at each target. Results are shown for a single fine-tuning seed. Dashed vertical lines indicate the conditioning targets and the grey density shows the band-gap distribution for the MP-20-OS test set.}
    \label{fig:band-gap}
    \vspace{-1\baselineskip}
\end{figure}

\section{Conclusions}
\label{sec:conclusion}

In this work, we introduced OxiGen, a diffusion model for oxidation-state-aware crystal generation that explicitly represents oxidation states during generation and enforces charge neutrality using dynamic programming over a finite-state automaton. Across experiments, OxiGen more faithfully reproduces the oxidation-state distributions of synthesised materials, achieves the highest rate of stable, unique, and novel crystals among evaluated methods, and maintains high compositional validity during band-gap-conditioned generation. The proposed species representation and structured output layer require only minimal modification to the underlying diffusion model architecture, so we expect they could be readily incorporated into other crystal generative models. Overall, our results demonstrate how incorporating prior chemical knowledge through simple heuristics and constraints can guide generative models towards more chemically plausible materials.

\noindent\textbf{Limitations.}\quad OxiGen inherits the limitations of formal integer oxidation states, which can be ambiguous or undefined for systems with substantial covalency, charge delocalisation, or metallic bonding \citep{karen_oxidation_2015,walsh_oxidation_2018}. As a result, OxiGen may mischaracterise or exclude valid materials for which a charge-neutral integer oxidation-state assignment is not chemically meaningful. OxiGen also relies on oxidation-state assignments inferred by our predictor, which may carry errors or biases. More broadly, charge neutrality is only one aspect of chemical validity: OxiGen does not enforce other constraints, such as electronegativity balance or consistency between an oxidation state and its local coordination and bonding environment. Future work could incorporate additional chemical constraints and explore more flexible inductive biases that retain the chemical information captured by oxidation states while moving beyond discrete formal assignments.

\newpage

\section*{AI Disclosure}

In this work, we used generative AI tools to assist with the research process, including formalising the methodology, writing code, and implementing baselines. Generative AI tools were also used to help identify relevant literature, edit parts of the manuscript for readability, and create figures. The authors proposed the research question, developed the core methodology and experimental design and analysed the results. We have reviewed all AI-assisted work and take responsibility for the final content of this work, including text, claims or artefacts produced with the aid of generative AI.

\section*{Reproducibility Statement}

We provide the information needed to reproduce our methodology and experiments. The diffusion model and structured inference procedure are described in Section~\ref{sec:method}, with further details in Appendix~\ref{app:oxigen}. Dataset curation and preprocessing are described in Appendix~\ref{app:os-pred}. Model architectures and hyperparameters are given in Appendices~\ref{app:os-models}, \ref{app:oxigen-train} and~\ref{app:band-gap}, and the evaluation procedure is described in Appendix~\ref{app:os-eval}. The source code and model checkpoints will be made publicly available upon publication.

\section*{Acknowledgements}

D.J., K.E.J., A.M.G. acknowledge the AI for Chemistry: AIchemy hub for funding (EPSRC grants EP/Y028775/1 and EP/Y028759/1). K.E.J. acknowledges a Royal Society Faraday Discovery Fellowship. The authors acknowledge the use of resources provided by the Isambard-AI National AI Research Resource (AIRR). Isambard-AI is operated by the University of Bristol and is funded by the UK Government’s Department for Science, Innovation and Technology (DSIT) via UK Research and Innovation; and the Science and Technology Facilities Council [ST/AIRR/I-A-I/1023].

\bibliographystyle{iclr2027_conference}
\bibliography{references}

@inproceedings{ahmed_semantic_2022,
	title = {Semantic {Probabilistic} {Layers} for {Neuro}-{Symbolic} {Learning}},
	volume = {35},
	url = {https://proceedings.neurips.cc/paper_files/paper/2022/hash/c182ec594f38926b7fcb827635b9a8f4-Abstract-Conference.html},
	doi = {10.52202/068431-2171},
	booktitle = {Advances in {Neural} {Information} {Processing} {Systems}},
	publisher = {Curran Associates, Inc.},
	author = {Ahmed, Kareem and Teso, Stefano and Chang, Kai-Wei and Van den Broeck, Guy and Vergari, Antonio},
	year = {2022},
	pages = {29944--29959},
}

@misc{ai4science_crystal-gfn_2023,
	title = {Crystal-{GFN}: sampling crystals with desirable properties and constraints},
	url = {https://arxiv.org/abs/2310.04925v1},
	doi = {10.48550/arXiv.2310.04925},
	publisher = {arXiv},
	author = {{Mila AI4Science} and Hernandez-Garcia, Alex and Duval, Alexandre and Volokhova, Alexandra and Bengio, Yoshua and Sharma, Divya and Carrier, Pierre Luc and Benabed, Yasmine and Koziarski, Michał and Schmidt, Victor},
	month = oct,
	year = {2023},
	note = {arXiv:2310.04925v1 [cs.LG]},
}

@inproceedings{austin_structured_2021,
	title = {Structured {Denoising} {Diffusion} {Models} in {Discrete} {State}-{Spaces}},
	volume = {34},
	url = {https://proceedings.neurips.cc/paper/2021/hash/958c530554f78bcd8e97125b70e6973d-Abstract.html},
	booktitle = {Advances in {Neural} {Information} {Processing} {Systems}},
	publisher = {Curran Associates, Inc.},
	author = {Austin, Jacob and Johnson, Daniel D. and Ho, Jonathan and Tarlow, Daniel and van den Berg, Rianne},
	year = {2021},
	pages = {17981--17993},
}

@inproceedings{cardei_constrained_2025,
	author = {Cardei, Michael and Christopher, Jacob K. and Kailkhura, Bhavya and Hartvigsen, Tom and Fioretto, Ferdinando},
	booktitle = {Advances in Neural Information Processing Systems},
	doi = {10.52202/085713-0415},
	editor = {D. Belgrave and C. Zhang and H. Lin and R. Pascanu and P. Koniusz and M. Ghassemi and N. Chen},
	pages = {12383--12416},
	publisher = {Curran Associates, Inc.},
	title = {Constrained Discrete Diffusion},
	url = {https://proceedings.neurips.cc/paper_files/paper/2025/file/124c1cc9f7750f512ed81cdbf2fd5bf5-Paper-Conference.pdf},
	volume = {38},
	year = {2025},
}

@inproceedings{chang_space_2025,
	title = {Space {Group} {Equivariant} {Crystal} {Diffusion}},
	volume = {38},
	url = {https://papers.nips.cc/paper_files/paper/2025/hash/697b2f31f99fb79f8a0a16e923b2471d-Abstract-Conference.html},
	doi = {10.52202/085713-2439},
	booktitle = {Advances in {Neural} {Information} {Processing} {Systems}},
	publisher = {Curran Associates, Inc.},
	author = {Chang, Rees and Pak, Angela and Guerra, Alex and Zhan, Ni and Richardson, Nick and Ertekin, Elif and Adams, Ryan P.},
	year = {2025},
	pages = {72772--72805},
}

@article{cheetham_artificial_2024,
	title = {Artificial {Intelligence} {Driving} {Materials} {Discovery}? {Perspective} on the {Article}: {Scaling} {Deep} {Learning} for {Materials} {Discovery}},
	volume = {36},
	url = {https://doi.org/10.1021/acs.chemmater.4c00643},
	doi = {10.1021/acs.chemmater.4c00643},
	number = {8},
	journal = {Chemistry of Materials},
	author = {Cheetham, Anthony K. and Seshadri, Ram},
	month = apr,
	year = {2024},
	pages = {3490--3495},
}

@article{cheng_enhancing_2026,
	title = {Enhancing materials discovery with valence-constrained design in generative modeling},
	url = {https://www.nature.com/articles/s43588-026-01037-2},
	doi = {10.1038/s43588-026-01037-2},
	journal = {Nature Computational Science},
	publisher = {Nature Publishing Group},
	author = {Cheng, Mouyang and Luo, Weiliang and Tang, Hao and Yu, Bowen and Cheng, Yongqiang and Xie, Weiwei and Li, Ju and Kulik, Heather J. and Li, Mingda},
	month = aug,
	year = {2026},
}

@inproceedings{christopher_neuro-symbolic_2025,
	title = {Neuro-{Symbolic} {Generative} {Diffusion} {Models} for {Physically} {Grounded}, {Robust}, and {Safe} {Generation}},
	url = {https://proceedings.mlr.press/v288/christopher25a.html},
	booktitle = {Proceedings of the International Conference on Neuro-symbolic Systems},
	publisher = {PMLR},
	author = {Christopher, Jacob K. and Cardei, Michael and Liang, Jinhao and Fioretto, Ferdinando},
	month = {28--30 May},
	year = {2025},
	pages = {188--213},
    volume = {288},
}

@misc{dang_constrained_2026,
	title = {Constrained {Decoding} for {Diffusion} {Language} {Models} via {Efficient} {Inference} over {Finite} {Automata}},
	url = {https://arxiv.org/abs/2607.07026},
	doi = {10.48550/arXiv.2607.07026},
	publisher = {arXiv},
	author = {Dang, Meihua and Ermon, Stefano},
	month = jul,
	year = {2026},
	note = {arXiv:2607.07026 [cs.LG]},
}

@article{davies_smact_2019,
	title = {{SMACT}: {Semiconducting} {Materials} by {Analogy} and {Chemical} {Theory}},
	volume = {4},
	url = {https://joss.theoj.org/papers/10.21105/joss.01361},
	doi = {10.21105/joss.01361},
	number = {38},
	journal = {Journal of Open Source Software},
	author = {Davies, Daniel W. and Butler, Keith T. and Jackson, Adam J. and Skelton, Jonathan M. and Morita, Kazuki and Walsh, Aron},
	month = jun,
	year = {2019},
	pages = {1361},
}

@inproceedings{dragone_neuro-symbolic_2021,
	series = {{CEUR} {Workshop} {Proceedings}},
	title = {Neuro-{Symbolic} {Constraint} {Programming} for {Structured} {Prediction}},
	volume = {2986},
	url = {https://ceur-ws.org/Vol-2986/},
	booktitle = {Proceedings of the 15th {International} {Workshop} on {Neural}-{Symbolic} {Learning} and {Reasoning}},
	publisher = {CEUR-WS.org},
	author = {Dragone, Paolo and Teso, Stefano and Passerini, Andrea},
	year = {2021},
	pages = {6--14},
}

@article{forney_viterbi_1973,
	title = {The {Viterbi} algorithm},
	volume = {61},
	url = {https://ieeexplore.ieee.org/document/1450960},
	doi = {10.1109/PROC.1973.9030},
	number = {3},
	journal = {Proceedings of the IEEE},
	author = {Forney, G.D.},
	month = mar,
	year = {1973},
	pages = {268--278},
}

@article{fu_composition_2023,
	title = {Composition {Based} {Oxidation} {State} {Prediction} of {Materials} {Using} {Deep} {Learning} {Language} {Models}},
	volume = {10},
	url = {https://onlinelibrary.wiley.com/doi/abs/10.1002/advs.202301011},
	doi = {10.1002/advs.202301011},
	number = {28},
	journal = {Advanced Science},
	author = {Fu, Nihang and Hu, Jeffrey and Feng, Ying and Morrison, Gregory and zur Loye, Hans-Conrad and Hu, Jianjun},
	year = {2023},
	pages = {2301011},
}

@inproceedings{gasteiger_gemnet_2021,
	title = {{GemNet}: {Universal} {Directional} {Graph} {Neural} {Networks} for {Molecules}},
	volume = {34},
	url = {https://proceedings.neurips.cc/paper/2021/hash/35cf8659cfcb13224cbd47863a34fc58-Abstract.html},
	booktitle = {Advances in {Neural} {Information} {Processing} {Systems}},
	publisher = {Curran Associates, Inc.},
	author = {Gasteiger, Johannes and Becker, Florian and Günnemann, Stephan},
	year = {2021},
	pages = {6790--6802},
}

@inproceedings{hollmer_open_2025,
	title = {Open {Materials} {Generation} with {Stochastic} {Interpolants}},
	url = {https://proceedings.mlr.press/v267/hollmer25a.html},
	booktitle = {Proceedings of the 42nd {International} {Conference} on {Machine} {Learning}},
	publisher = {PMLR},
	volume = {267},
	author = {Höllmer, Philipp and Egg, Thomas and Martirossyan, Maya and Fuemmeler, Eric and Shui, Zeren and Gupta, Amit and Prakash, Pawan and Roitberg, Adrian and Liu, Mingjie and Karypis, George and Transtrum, Mark and Hennig, Richard and Tadmor, Ellad B. and Martiniani, Stefano},
	month = jul,
	year = {2025},
	pages = {23417--23450},
}

@inproceedings{hoernle_multiplexnet_2022,
	title = {{MultiplexNet}: {Towards} {Fully} {Satisfied} {Logical} {Constraints} in {Neural} {Networks}},
	url = {https://www.research.ed.ac.uk/en/publications/multiplexnet-towards-fully-satisfied-logical-constraints-in-neura/},
	doi = {10.1609/aaai.v36i5.20512},
	booktitle = {Proceedings of the {AAAI} Conference on Artificial Intelligence},
	volume = {36(5)},
	publisher = {Association for the Advancement of Artificial Intelligence},
	author = {Hoernle, Nick and Karampatsis, Rafael and Belle, Vaishak and Gal, Yakov Kobi},
	month = jun,
	year = {2022},
	pages = {5700--5709},
}

@article{jain_commentary_2013,
	title = {Commentary: {The} {Materials} {Project}: {A} materials genome approach to accelerating materials innovation},
	volume = {1},
	url = {https://doi.org/10.1063/1.4812323},
	doi = {10.1063/1.4812323},
	number = {1},
	journal = {APL Materials},
	author = {Jain, Anubhav and Ong, Shyue Ping and Hautier, Geoffroy and Chen, Wei and Richards, William Davidson and Dacek, Stephen and Cholia, Shreyas and Gunter, Dan and Skinner, David and Ceder, Gerbrand and Persson, Kristin A.},
	month = jul,
	year = {2013},
	pages = {011002},
}

@inproceedings{jiao_crystal_2023,
	title = {Crystal {Structure} {Prediction} by {Joint} {Equivariant} {Diffusion}},
	volume = {36},
	url = {https://proceedings.neurips.cc/paper_files/paper/2023/hash/38b787fc530d0b31825827e2cc306656-Abstract-Conference.html},
	doi = {10.52202/075280-0767},
	booktitle = {Advances in {Neural} {Information} {Processing} {Systems}},
	publisher = {Curran Associates, Inc.},
	author = {Jiao, Rui and Huang, Wenbing and Lin, Peijia and Han, Jiaqi and Chen, Pin and Lu, Yutong and Liu, Yang},
	year = {2023},
	pages = {17464--17497},
}

@inproceedings{jiao_space_2024,
	title = {Space {Group} {Constrained} {Crystal} {Generation}},
	booktitle = {The Twelfth International Conference on Learning Representations},
	url = {https://proceedings.iclr.cc/paper_files/paper/2024/hash/1ba1b2133ac3d3dda0ce5cd67dab8456-Abstract-Conference.html},
	journal = {International Conference on Learning Representations},
	author = {Jiao, Rui and Huang, Wenbing and Liu, Yu and Zhao, Deli and Liu, Yang},
	month = may,
	year = {2024},
	pages = {6836--6853},
}

@inproceedings{joshi_all-atom_2025,
	title = {All-atom {Diffusion} {Transformers}: {Unified} generative modelling of molecules and materials},
	url = {https://proceedings.mlr.press/v267/joshi25a.html},
	booktitle = {Proceedings of the 42nd {International} {Conference} on {Machine} {Learning}},
	publisher = {PMLR},
	volume = {267},
	author = {Joshi, Chaitanya K. and Fu, Xiang and Liao, Yi-Lun and Gharakhanyan, Vahe and Miller, Benjamin Kurt and Sriram, Anuroop and Ulissi, Zachary Ward},
	month = jul,
	year = {2025},
	pages = {28393--28417},
}

@article{karen_oxidation_2015,
	title = {Oxidation {State}, {A} {Long}-{Standing} {Issue}!},
	volume = {54},
	url = {https://onlinelibrary.wiley.com/doi/abs/10.1002/anie.201407561},
	doi = {10.1002/anie.201407561},
	number = {16},
	journal = {Angewandte Chemie International Edition},
	author = {Karen, Pavel},
	year = {2015},
	pages = {4716--4726},
}

@inproceedings{van_krieken_neurosymbolic_2025,
	title = {Neurosymbolic {Diffusion} {Models}},
	volume = {38},
	url = {https://papers.nips.cc/paper_files/paper/2025/hash/c60bd92a01804b7df0540ed7ca2f7c05-Abstract-Conference.html},
	doi = {10.52202/085713-4535},
	booktitle = {Advances in {Neural} {Information} {Processing} {Systems}},
	publisher = {Curran Associates, Inc.},
	author = {van Krieken, Emile and Minervini, Pasquale and Ponti, Edoardo Maria and Vergari, Antonio},
	year = {2025},
	pages = {135889--135931},
}

@inproceedings{levy_symmcd_2024,
	title = {{SymmCD}: {Symmetry}-{Preserving} {Crystal} {Generation} with {Diffusion} {Models}},
	url = {https://openreview.net/forum?id=xnssGv9rpW},
	booktitle = {The Thirteenth International Conference on Learning Representations},
	author = {Levy, Daniel and Panigrahi, Siba Smarak and Kaba, Sékou-Oumar and Zhu, Qiang and Lee, Kin Long Kelvin and Galkin, Mikhail and Miret, Santiago and Ravanbakhsh, Siamak},
	year = {2025},
}

@article{luo_crystalflow_2025,
	title = {{CrystalFlow}: a flow-based generative model for crystalline materials},
	volume = {16},
	url = {https://www.nature.com/articles/s41467-025-64364-4},
	doi = {10.1038/s41467-025-64364-4},
	number = {1},
	journal = {Nature Communications},
	publisher = {Nature Publishing Group},
	author = {Luo, Xiaoshan and Wang, Zhenyu and Wang, Qingchang and Shao, Xuechen and Lv, Jian and Wang, Lei and Wang, Yanchao and Ma, Yanming},
	month = oct,
	year = {2025},
	pages = {9267},
}

@misc{mastej_chemical_2026,
	title = {Chemical filters for ultra-high-throughput materials screening and generation},
	url = {https://arxiv.org/abs/2607.17910},
	doi = {10.48550/arXiv.2607.17910},
	publisher = {arXiv},
	author = {Mastej, Kinga O. and Detrattanawichai, Panyalak and Park, Hyunsoo and Onwuli, Anthony and Negishi, Masahiro and Walsh, Aron},
	month = jul,
	year = {2026},
	note = {arXiv:2607.17910 [cond-mat.mtrl-sci]},
}

@article{merchant_scaling_2023,
	title = {Scaling deep learning for materials discovery},
	volume = {624},
	url = {https://www.nature.com/articles/s41586-023-06735-9},
	doi = {10.1038/s41586-023-06735-9},
	number = {7990},
	journal = {Nature},
	publisher = {Nature Publishing Group},
	author = {Merchant, Amil and Batzner, Simon and Schoenholz, Samuel S. and Aykol, Muratahan and Cheon, Gowoon and Cubuk, Ekin Dogus},
	month = dec,
	year = {2023},
	pages = {80--85},
}

@inproceedings{miller_flowmm_2024,
	title = {{FlowMM}: {Generating} {Materials} with {Riemannian} {Flow} {Matching}},
	url = {https://proceedings.mlr.press/v235/miller24a.html},
	booktitle = {Proceedings of the 41st {International} {Conference} on {Machine} {Learning}},
	publisher = {PMLR},
	volume = {235},
	author = {Miller, Benjamin Kurt and Chen, Ricky T. Q. and Sriram, Anuroop and Wood, Brandon M.},
	month = jul,
	year = {2024},
	pages = {35664--35686},
}

@article{ong_python_2013,
	title = {Python {Materials} {Genomics} (pymatgen): {A} robust, open-source python library for materials analysis},
	volume = {68},
	url = {https://www.sciencedirect.com/science/article/pii/S0927025612006295},
	doi = {10.1016/j.commatsci.2012.10.028},
	journal = {Computational Materials Science},
	author = {Ong, Shyue Ping and Richards, William Davidson and Jain, Anubhav and Hautier, Geoffroy and Kocher, Michael and Cholia, Shreyas and Gunter, Dan and Chevrier, Vincent L. and Persson, Kristin A. and Ceder, Gerbrand},
	month = feb,
	year = {2013},
	pages = {314--319},
}

@article{park_guiding_2026,
	title = {Guiding generative models to uncover diverse and novel crystals via reinforcement learning},
	volume = {8},
	url = {https://www.nature.com/articles/s42256-026-01262-4},
	doi = {10.1038/s42256-026-01262-4},
	number = {7},
	journal = {Nature Machine Intelligence},
	publisher = {Nature Publishing Group},
	author = {Park, Hyunsoo and Walsh, Aron},
	month = jul,
	year = {2026},
	pages = {1087--1099},
}

@article{park_has_2024,
	title = {Has generative artificial intelligence solved inverse materials design?},
	volume = {7},
	url = {https://www.cell.com/matter/abstract/S2590-2385(24)00242-X},
	doi = {10.1016/j.matt.2024.05.017},
	number = {7},
	journal = {Matter},
	publisher = {Elsevier},
	author = {Park, Hyunsoo and Li, Zhenzhu and Walsh, Aron},
	month = jul,
	year = {2024},
	pages = {2355--2367},
}

@article{park_closing_2026,
	title = {Closing the synthesis gap in computational materials design},
	volume = {5},
	url = {https://www.nature.com/articles/s44160-026-01027-2},
	doi = {10.1038/s44160-026-01027-2},
	number = {5},
	journal = {Nature Synthesis},
	publisher = {Nature Publishing Group},
	author = {Park, Hyunsoo and Mastej, Kinga O. and Detrattanawichai, Panyalak and Nduma, Ryan and Walsh, Aron},
	month = may,
	year = {2026},
	pages = {651--659},
}

@article{rabiner_tutorial_1989,
	title = {A tutorial on hidden {Markov} models and selected applications in speech recognition},
	volume = {77},
	url = {https://ieeexplore.ieee.org/document/18626},
	doi = {10.1109/5.18626},
	number = {2},
	journal = {Proceedings of the IEEE},
	author = {Rabiner, L.R.},
	month = feb,
	year = {1989},
	pages = {257--286},
}

@inproceedings{sahoo_simple_2024,
	title = {Simple and {Effective} {Masked} {Diffusion} {Language} {Models}},
	volume = {37},
	url = {https://proceedings.neurips.cc/paper_files/paper/2024/hash/eb0b13cc515724ab8015bc978fdde0ad-Abstract-Conference.html},
	doi = {10.52202/079017-4135},
	booktitle = {Advances in {Neural} {Information} {Processing} {Systems}},
	publisher = {Curran Associates, Inc.},
	author = {Sahoo, Subham Sekhar and Arriola, Marianne and Schiff, Yair and Gokaslan, Aaron and Marroquin, Edgar and Chiu, Justin T. and Rush, Alexander and Kuleshov, Volodymyr},
	year = {2024},
	pages = {130136--130184},
}

@inproceedings{suresh_dingo_2025,
	title = {{DINGO}: {Constrained} {Inference} for {Diffusion} {LLMs}},
	volume = {38},
	url = {https://proceedings.neurips.cc/paper_files/paper/2025/hash/eb17a2030d1bd4a1bd29531bcd626705-Abstract-Conference.html},
	doi = {10.52202/085713-5362},
	booktitle = {Advances in {Neural} {Information} {Processing} {Systems}},
	publisher = {Curran Associates, Inc.},
	author = {Suresh, Tarun and Banerjee, Debangshu and Ugare, Shubham and Misailovic, Sasa and Singh, Gagandeep},
	year = {2025},
	pages = {160518--160551},
}

@misc{tomasi_primal-dual_2026,
	title = {Primal-{Dual} {Guided} {Decoding} for {Constrained} {Discrete} {Diffusion}},
	url = {https://arxiv.org/abs/2605.09749},
	doi = {10.48550/arXiv.2605.09749},
	publisher = {arXiv},
	author = {Tomasi, Federico and Moor, Dmitrii and Wang, Alice and Lalmas, Mounia},
	month = may,
	year = {2026},
	note = {arXiv:2605.09749 [cs.AI]},
}

@misc{veljkovic_crystalite_2026,
	title = {Crystalite: {A} {Lightweight} {Transformer} for {Efficient} {Crystal} {Modeling}},
	url = {https://arxiv.org/abs/2604.02270},
	doi = {10.48550/arXiv.2604.02270},
	publisher = {arXiv},
	author = {Veljković, Tin Hadži and Rosenthal, Joshua and Lončarić, Ivor and van de Meent, Jan-Willem},
	month = jul,
	year = {2026},
	note = {arXiv:2604.02270 [cs.LG]},
}

@article{walsh_oxidation_2018,
	title = {Oxidation states and ionicity},
	volume = {17},
	url = {https://www.nature.com/articles/s41563-018-0165-7},
	doi = {10.1038/s41563-018-0165-7},
	number = {11},
	journal = {Nature Materials},
	publisher = {Nature Publishing Group},
	author = {Walsh, Aron and Sokol, Alexey A. and Buckeridge, John and Scanlon, David O. and Catlow, C. Richard A.},
	month = nov,
	year = {2018},
	pages = {958--964},
}

@inproceedings{xie_crystal_2022,
	title = {Crystal Diffusion Variational Autoencoder for Periodic Material Generation},
	author = {Tian Xie and Xiang Fu and Octavian-Eugen Ganea and Regina Barzilay and Tommi S. Jaakkola},
	booktitle = {International Conference on Learning Representations},
	year = {2022},
	url = {https://openreview.net/forum?id=03RLpj-tc_},
}

@inproceedings{yan_complete_2024,
	title = {Complete and {Efficient} {Graph} {Transformers} for {Crystal} {Material} {Property} {Prediction}},
	url = {https://proceedings.iclr.cc/paper_files/paper/2024/hash/0ab51646ca369140c3c3ece011b66587-Abstract-Conference.html},
	booktitle = {International {Conference} on {Learning} {Representations}},
	author = {Yan, Keqiang and Fu, Cong and Qian, Xiaofeng and Qian, Xiaoning and Ji, Shuiwang},
	month = may,
	year = {2024},
	pages = {2564--2590},
}

@misc{yang_mattersim_2024,
	title = {{MatterSim}: {A} {Deep} {Learning} {Atomistic} {Model} {Across} {Elements}, {Temperatures} and {Pressures}},
	url = {https://arxiv.org/abs/2405.04967},
	doi = {10.48550/arXiv.2405.04967},
	publisher = {arXiv},
	author = {Yang, Han and Hu, Chenxi and Zhou, Yichi and Liu, Xixian and Shi, Yu and Li, Jielan and Li, Guanzhi and Chen, Zekun and Chen, Shuizhou and Zeni, Claudio and Horton, Matthew and Pinsler, Robert and Fowler, Andrew and Zügner, Daniel and Xie, Tian and Smith, Jake and Sun, Lixin and Wang, Qian and Kong, Lingyu and Liu, Chang and Hao, Hongxia and Lu, Ziheng},
	month = may,
	year = {2024},
	note = {arXiv:2405.04967 [cond-mat.mtrl-sci]},
}

@article{yao_machine_2023,
	title = {Machine learning for a sustainable energy future},
	volume = {8},
	url = {https://www.nature.com/articles/s41578-022-00490-5},
	doi = {10.1038/s41578-022-00490-5},
	number = {3},
	journal = {Nature Reviews Materials},
	publisher = {Nature Publishing Group},
	author = {Yao, Zhenpeng and Lum, Yanwei and Johnston, Andrew and Mejia-Mendoza, Luis Martin and Zhou, Xin and Wen, Yonggang and Aspuru-Guzik, Alán and Sargent, Edward H. and Seh, Zhi Wei},
	month = mar,
	year = {2023},
	pages = {202--215},
}

@article{zagorac_recent_2019,
	title = {Recent developments in the {Inorganic} {Crystal} {Structure} {Database}: theoretical crystal structure data and related features},
	volume = {52},
	url = {https://journals.iucr.org/j/issues/2019/05/00/in5024/},
	doi = {10.1107/S160057671900997X},
	number = {5},
	journal = {Journal of Applied Crystallography},
	publisher = {International Union of Crystallography},
	author = {Zagorac, D. and Müller, H. and Ruehl, S. and Zagorac, J. and Rehme, S.},
	month = oct,
	year = {2019},
	pages = {918--925},
}

@article{zeni_generative_2025,
	title = {A generative model for inorganic materials design},
	volume = {639},
	url = {https://www.nature.com/articles/s41586-025-08628-5},
	doi = {10.1038/s41586-025-08628-5},
	number = {8055},
	journal = {Nature},
	publisher = {Nature Publishing Group},
	author = {Zeni, Claudio and Pinsler, Robert and Zügner, Daniel and Fowler, Andrew and Horton, Matthew and Fu, Xiang and Wang, Zilong and Shysheya, Aliaksandra and Crabbé, Jonathan and Ueda, Shoko and Sordillo, Roberto and Sun, Lixin and Smith, Jake and Nguyen, Bichlien and Schulz, Hannes and Lewis, Sarah and Huang, Chin-Wei and Lu, Ziheng and Zhou, Yichi and Yang, Han and Hao, Hongxia and Li, Jielan and Yang, Chunlei and Li, Wenjie and Tomioka, Ryota and Xie, Tian},
	month = mar,
	year = {2025},
	pages = {624--632},
}

@article{yi_crystaldit_2026,
	title = {{CrystalDiT}: {Simple} {Diffusion} {Transformers} for {Crystal} {Generation}},
	volume = {40},
	issn = {2374-3468},
	shorttitle = {{CrystalDiT}},
	url = {https://ojs.aaai.org/index.php/AAAI/article/view/37121},
	doi = {10.1609/aaai.v40i2.37121},
	language = {en},
	number = {2},
	urldate = {2026-09-18},
	journal = {Proceedings of the AAAI Conference on Artificial Intelligence},
	author = {Yi, Xiaohan and Xu, Guikun and Zhang, Zhong and Liu, Liu and Bian, Yatao and Xiao, Xi and Zhao, Peilin},
	month = mar,
	year = {2026},
	pages = {1462--1470},
}

@article{park_generative_2026,
	title = {Generative modelling of inorganic materials with explicit electronic structure},
	volume = {17},
	copyright = {2026 The Author(s)},
	issn = {2041-1723},
	url = {https://www.nature.com/articles/s41467-026-73985-2},
	doi = {10.1038/s41467-026-73985-2},
	number = {1},
	urldate = {2026-09-18},
	journal = {Nature Communications},
	publisher = {Nature Publishing Group},
	author = {Park, Junkil and Choi, Junyoung and Jung, Yousung},
	month = jun,
	year = {2026},
	pages = {7070},
}

@misc{betala_lematgenbench_2026,
      title={{LeMat-GenBench}: A Unified Evaluation Framework for Crystal Generative Models}, 
      author={Siddharth Betala and Samuel P. Gleason and Ali Ramlaoui and Andy Xu and Georgia Channing and Daniel Levy and Clémentine Fourrier and Nikita Kazeev and Chaitanya K. Joshi and Sékou-Oumar Kaba and Félix Therrien and Alex Hernandez-Garcia and Rocío Mercado and N. M. Anoop Krishnan and Alexandre Duval},
      year={2025},
      eprint={2512.04562},
      archivePrefix={arXiv},
      primaryClass={cs.LG},
      url={https://arxiv.org/abs/2512.04562}, 
}

@inproceedings{kazeev_wyckoff_2025,
	title = {Wyckoff {Transformer}: {Generation} of {Symmetric} {Crystals}},
	url = {https://proceedings.mlr.press/v267/kazeev25a.html},
	booktitle = {Proceedings of the 42nd {International} {Conference} on {Machine} {Learning}},
	publisher = {PMLR},
	volume = {267},
	author = {Kazeev, Nikita and Nong, Wei and Romanov, Ignat and Zhu, Ruiming and Ustyuzhanin, Andrey E. and Yamazaki, Shuya and Hippalgaonkar, Kedar},
	month = jul,
	year = {2025},
	pages = {29495--29526},
}

@inproceedings{xu_plaid_2026,
	title = {{PLaID}++: {A} {Preference} {Aligned} {Language} {Model} for {Targeted} {Inorganic} {Materials} {Design}},
	url = {https://arxiv.org/abs/2509.07150},
	booktitle = {Proceedings of the 43rd {International} {Conference} on {Machine} {Learning}},
	publisher = {PMLR},
	volume = {306},
	author = {Xu, Andy and Desai, Rohan and Wang, Larry and Ritz, Ethan and Hope, Gabriel},
	month = jul,
	year = {2026},
}

@article{cao_space_2025,
	title = {Space group informed transformer for crystalline materials generation},
	volume = {70},
	url = {https://www.sciencedirect.com/science/article/pii/S2095927325009752},
	doi = {10.1016/j.scib.2025.09.035},
	number = {21},
	journal = {Science Bulletin},
	author = {Cao, Zhendong and Luo, Xiaoshan and Lv, Jian and Wang, Lei},
	month = nov,
	year = {2025},
	pages = {3522--3533},
}

@inproceedings{martirossyan_framework_2026,
	title = {Framework-Constrained Materials Generation},
	url = {https://openreview.net/forum?id=Yk4FcrVcU2},
	booktitle = {{AI} for Accelerated Materials Design Workshop at {ICLR} 2026},
	author = {Martirossyan, Maya M. and Pan, Hillary and Höllmer, Philipp and Martiniani, Stefano},
	year = {2026},
}

@misc{okhotin_miad_2025,
	title = {{MiAD}: Mirage Atom Diffusion for De Novo Crystal Generation},
	url = {https://arxiv.org/abs/2511.14426},
	doi = {10.48550/arXiv.2511.14426},
	publisher = {arXiv},
	author = {Okhotin, Andrey and Nakhodnov, Maksim and Kazeev, Nikita and Lazarev, Mikhail and Ustyuzhanin, Andrey E. and Vetrov, Dmitry},
	month = nov,
	year = {2025},
	note = {arXiv:2511.14426 [cs.LG]},
}

@misc{seong_multimodal_2026,
	title = {Multimodal Crystal Flow: Any-to-Any Modality Generation for Unified Crystal Modeling},
	url = {https://arxiv.org/abs/2602.20210},
	doi = {10.48550/arXiv.2602.20210},
	publisher = {arXiv},
	author = {Seong, Kiyoung and Ahn, Sungsoo and Han, Sehui and Park, Changyoung},
	month = feb,
	year = {2026},
	note = {arXiv:2602.20210 [cs.LG]},
}

@article{ahlawat_family_2026,
	title = {A family of large language models for materials research with insights into model adaptability in continued pretraining},
	volume = {8},
	url = {https://www.nature.com/articles/s42256-026-01199-8},
	doi = {10.1038/s42256-026-01199-8},
	journal = {Nature Machine Intelligence},
	author = {Ahlawat, Dhruv and Mishra, Vaibhav and Singh, Somaditya and Zaki, Mohd and Bihani, Vaibhav and Grover, Hargun Singh and Mishra, Biswajit and Miret, Santiago and {Mausam} and Krishnan, N. M. Anoop},
	year = {2026},
	pages = {435--448},
}

@misc{bone_discovery_2025,
	title = {Discovery and recovery of crystalline materials with property-conditioned transformers},
	url = {https://arxiv.org/abs/2511.21299},
	doi = {10.48550/arXiv.2511.21299},
	publisher = {arXiv},
	author = {Bone, Cyprien and Walker, Matthew and Leng, Kuangdai and Antunes, Luis M. and Grau-Crespo, Ricardo and Aligayev, Amil and Dominguez, Javier and Butler, Keith T.},
	month = nov,
	year = {2025},
	note = {arXiv:2511.21299 [cond-mat.mtrl-sci]},
}

@article{park_exploration_2025,
	title = {Exploration of crystal chemical space using text-guided generative artificial intelligence},
	volume = {16},
	url = {https://www.nature.com/articles/s41467-025-59636-y},
	doi = {10.1038/s41467-025-59636-y},
	number = {1},
	journal = {Nature Communications},
	author = {Park, Hyunsoo and Onwuli, Anthony and Walsh, Aron},
	year = {2025},
	pages = {4379},
}

\newpage

\appendix

\section{Structured Output Layer}
\label{app:structured}

\noindent\textbf{Proof of Proposition~\ref{prop:map}.}\quad
We show by induction on $i$ that $V_i(c)$ is the maximum unnormalised weight of any assignment to the first $i$ sites with cumulative charge $c$. The claim holds for $i=0$ since $V_0(c)=\ind[c=0]$. Assuming it holds for $i-1$, any assignment to the first $i$ sites with cumulative charge $c$ can be decomposed into an assignment to the first $i-1$ sites with cumulative charge $c-z(s_{0,i})$ and a choice of $s_{0,i}$. Hence,
\begin{equation}
    V_i(c)=\max_{s_{0,i}\in\mathcal{S}}p_{\theta,i}(s_{0,i}\mid\boldsymbol{M}_t,t)\,V_{i-1}\!\left(c-z(s_{0,i})\right),
\label{eq:map-proof}
\end{equation}
which is therefore the maximum unnormalised weight among all such
assignments.

It follows that $V_N(0)$ is the maximum unnormalised weight over all charge-neutral assignments. Since the normalising constant $Z_\theta(\boldsymbol{M}_t,t)$ does not depend on the assignment, maximising the unnormalised weight is equivalent to maximising $\widetilde p_{\theta}(\boldsymbol{S}_0\mid\boldsymbol{M}_t,t)$ over its support. Backtracking through the stored maximisers therefore recovers the exact MAP charge-neutral assignment. \hfill$\square$

\noindent\textbf{Proof of Proposition~\ref{prop:sampling}.}\quad Let $\boldsymbol{S}_0=(s_{0,1},\ldots,s_{0,N})$ be a charge-neutral species assignment, with cumulative charges $c_i=\sum_{j=1}^i z(s_{0,j})$ and $c_0=0$. Under the sequential sampling rule in Eq.~\eqref{eq:sampler}, its probability is
\begin{align}
\prod_{i=1}^N
\frac{
p_{\theta,i}(s_{0,i}\mid\boldsymbol{M}_t,t)\,
B_i(c_i)
}{
B_{i-1}(c_{i-1})
}
&=
\frac{
\prod_{i=1}^N p_{\theta,i}(s_{0,i}\mid\boldsymbol{M}_t,t)\,
B_N(c_N)
}{
B_0(c_0)
}
\label{eq:sampling-telescope}
\\
&=
\frac{
\prod_{i=1}^N p_{\theta,i}(s_{0,i}\mid\boldsymbol{M}_t,t)
}{
Z_\theta(\boldsymbol{M}_t,t)
},
\label{eq:sampling-proof}
\end{align}
where the backward-weight factors telescope. Since $\boldsymbol{S}_0$ is charge neutral, $c_N=0$, so $B_N(c_N)=B_N(0)=1$, while $c_0=0$ and $B_0(0)=Z_\theta(\boldsymbol{M}_t,t)$.

Thus the sequential procedure assigns probability $\widetilde p_\theta(\boldsymbol{S}_0\mid\boldsymbol{M}_t,t)$ to every charge-neutral species assignment. Any non-neutral assignment has probability zero because $B_N(c_N)=0$ whenever $c_N\neq 0$. Therefore the procedure samples exactly from the structured distribution. \hfill$\square$

\section{Oxidation-State Prediction Details}
\label{app:os-pred}

Oxidation-state prediction assigns an oxidation state to each atomic site of a crystal with known composition and structure. Given a crystal $\boldsymbol{M}=(\boldsymbol{A},\boldsymbol{F},\boldsymbol{L})$ with element identities $\boldsymbol{A}=(a_1,\ldots,a_N)\in\mathcal{A}^N$, we predict an oxidation-state assignment $\boldsymbol{Z}=(z_1,\ldots,z_N)\in\mathcal{Z}^N$, where $\mathcal{Z}$ is a vocabulary of integer oxidation states. A predictor defines a factorised distribution
\begin{equation}
p_{\phi}(\boldsymbol{Z}\mid\boldsymbol{M})
=
\prod_{i=1}^N
p_{\phi,i}(z_i\mid\boldsymbol{M}),
\end{equation}
where each site is predicted over the full oxidation-state vocabulary $\mathcal{Z}$. The predictor is trained using the categorical cross-entropy loss,
\begin{equation}
\mathcal{L}_{\mathrm{CE}}(\phi)
=
-\frac1N\sum_{i=1}^N
\log p_{\phi,i}(z_i^\star\mid\boldsymbol{M}),
\end{equation}
where $z_i^\star$ denotes the ground-truth oxidation state at site $i$. At inference, unconstrained predictions are obtained independently at each site by argmax decoding,
\begin{equation}
\hat z_i
=
\arg\max_{z\in\mathcal{Z}}
p_{\phi,i}(z\mid\boldsymbol{M}).
\end{equation}
Argmax decoding does not enforce that the predicted oxidation state is chemically admissible for the corresponding element or that the complete assignment is charge neutral.

\subsection{Vocabulary and Dataset Construction}
\label{app:os-vocab}

We define the oxidation-state vocabulary using the \texttt{SMACT}~v4.0.0 ICSD-24 oxidation-state filter with \texttt{consensus=6}, \texttt{commonality=low}, and \texttt{include\_zero=False} \citep{davies_smact_2019}. We use a stricter consensus threshold than the \texttt{SMACT} default of \texttt{consensus=3} to exclude extremely unusual oxidation states from the vocabulary. The admissible element--oxidation-state pairs cover 90 elements with oxidation states drawn from
\begin{equation}
\mathcal{Z}
=
\{-5,-4,-3,-2,-1,+1,+2,+3,+4,+5,+6,+7,+8\},
\label{eq:os-vocab}
\end{equation}
with an admissible subset $\mathcal{Z}(a)\subseteq\mathcal{Z}$ for each element $a$.

We construct an oxidation-state-labelled crystal structure dataset from the Inorganic Crystal Structure Database (ICSD) \citep{zagorac_recent_2019}. We apply the following curation procedure:
\begin{enumerate}
    \item Retain only structures for which every atomic site has an integer oxidation-state label, the complete assignment is charge neutral, and each site label belongs to the corresponding element-specific admissible set $\mathcal{Z}(a)$.
    \item Remove duplicate structures and structures with more than 200 atoms.
    \item Convert all retained structures to primitive Niggli-reduced cells.
    \item Define the set of elements with known mixed-valence chemistry, $\mathcal{A}_{\mathrm{mv}}$, using \texttt{SMACT}'s \texttt{MIXED\_VALENCE\_ELEMENTS} set, and exclude structures in which any element outside this set is assigned more than one oxidation state.
\end{enumerate}
We then group structures by reduced composition and construct composition-disjoint training, validation, and test splits such that no reduced composition occurs in more than one split. The resulting dataset contains 22,879 training, 2,831 validation, and 2,782 test structures.

\subsection{Model Architectures}
\label{app:os-models}

\noindent\textbf{BERTOS.}\quad BERTOS is a composition-only oxidation-state predictor based on a BERT transformer \citep{fu_composition_2023}. A composition is expanded into an electronegativity-ordered sequence of element tokens, which is processed by the transformer to produce contextual representations for each site. A shared classifier predicts an oxidation state for each token. The encoder is first pre-trained using masked-language modelling over element tokens and is then fine-tuned for oxidation-state prediction.

\noindent\textbf{CrystaliteOS.}\quad Our predictor, which we refer to as CrystaliteOS, is adapted from the Crystalite transformer \citep{veljkovic_crystalite_2026}. We retain the geometry-enhanced transformer backbone, remove the diffusion-time conditioning and prediction heads, and replace them with a linear per-site classifier over the oxidation-state vocabulary $\mathcal{Z}$. Unlike BERTOS, the predictor takes the complete crystal structure $\boldsymbol{M}=(\boldsymbol{A},\boldsymbol{F},\boldsymbol{L})$ as input.

Architecture and training settings are given in Table~\ref{tab:os-settings}. For each decoding method, we select the checkpoint with the highest validation crystal accuracy under that same decoding method. We therefore report argmax and Viterbi results using their respective best validation checkpoints.

\subsection{Structured Inference}
\label{app:grouped-dp}

For CrystaliteOS, we apply the structured Viterbi decoding of Section~\ref{sec:inference} to oxidation-state prediction. Given the factorised predictive distribution
\begin{equation}
p_{\phi}(\boldsymbol{Z}\mid\boldsymbol{M})
=
\prod_{i=1}^N p_{\phi,i}(z_i\mid\boldsymbol{M}),
\end{equation}
we define the structured distribution
\begin{equation}
\widetilde{p}_{\phi}(\boldsymbol{Z}\mid\boldsymbol{M})
\propto
\ind\!\left[\sum_{i=1}^N z_i=0\right]
\prod_{i=1}^N
\ind\!\left[z_i\in\mathcal{Z}(a_i)\right]
p_{\phi,i}(z_i\mid\boldsymbol{M}),
\label{eq:os-structured}
\end{equation}
whose support is restricted to charge-neutral assignments that use an admissible oxidation state for each element. The same dynamic program described in Section~\ref{sec:inference} can be used for exact inference under this distribution.

For CrystaliteOS, we additionally restrict mixed valence to elements in $\mathcal{A}_{\mathrm{mv}}$. For each element $a\notin\mathcal{A}_{\mathrm{mv}}$, all sites occupied by $a$ are constrained to have the same oxidation state. This restriction can be incorporated exactly into the dynamic program by grouping sites of the same element. Let
\begin{equation}
I_a=\{i:a_i=a\}
\label{eq:element-sites}
\end{equation}
denote the set of sites occupied by element $a$, and let $n_a=|I_a|$ denote its multiplicity. For an admissible oxidation state $z\in\mathcal{Z}(a)$, assigning oxidation state $z$ to every site in $I_a$ has unnormalised weight
\begin{equation}
w_a(z)
=
\prod_{i\in I_a}
p_{\phi,i}(z\mid\boldsymbol{M}),
\label{eq:group-weight}
\end{equation}
and contributes $n_a z$ to the cumulative charge. For each $a\notin\mathcal{A}_{\mathrm{mv}}$, the $n_a$ site-level transitions can therefore be replaced by a single transition with weight $w_a(z)$ and charge increment $n_a z$. This preserves the unnormalised weight of every assignment satisfying the single-valence constraint, so the forward--backward and Viterbi recursions apply unchanged over the restricted feasible set. Elements in $\mathcal{A}_{\mathrm{mv}}$ are not grouped and instead retain one dynamic-programming step per site, allowing different sites of the same element to adopt different oxidation states. We use Viterbi decoding to return the MAP assignment satisfying element-wise admissibility, charge neutrality, and the mixed-valence constraint.

\subsection{MP-20 oxidation-state labelling}
\label{app:mp20-labelling}

We use CrystaliteOS with charge-neutral Viterbi decoding to annotate MP-20 and apply the following curation procedure:
\begin{enumerate}
    \item Convert all structures to Niggli-reduced cells.
    \item Remove structures containing elements outside the oxidation-state vocabulary.
    \item Assign oxidation state zero to every site for structures in which all elements are classified as metals by \texttt{pymatgen}, following the conventional treatment of metallic alloys.
    \item Remove non-metallic elemental structures. Oxidation state zero is typically only assigned to non-metals in elemental structures, which constitute 0.2\% of the dataset; excluding them avoids expanding the species vocabulary to accommodate these cases.
    \item Assign oxidation states to all remaining structures using CrystaliteOS.
    \item Remove structures with no feasible charge-neutral assignment, or with a chemically implausible predicted assignment in which the same element adopts more than two distinct oxidation states or both positive and negative oxidation states.
\end{enumerate}
Table~\ref{tab:mp20-filter} summarises the filtering pipeline. Starting from 45,229 MP-20 structures, the final MP-20-OS dataset contains 42,576 structures, comprising 25,536 training, 8,495 validation, and 8,545 test structures.

\begin{table}[ht]
\centering
\small
\caption{MP-20-OS filtering. Rows under \textit{removed} report structures excluded at each filtering stage.}
\vspace{0.5\baselineskip}
\label{tab:mp20-filter}
\begin{tabular}{lrrrr}
\toprule
Stage & Train & Val. & Test & Total \\
\midrule
MP-20                              & 27,136 & 9,047 & 9,046 & 45,229 \\
\midrule
\multicolumn{5}{l}{\textit{Removed}} \\
Outside vocabulary                 &    289 &    90 &    83 &    462 \\
Elemental non-metal                &     55 &    20 &    10 &     85 \\
No feasible charge-neutral assignment
                                   &  1,003 &   359 &   327 &  1,689 \\
Chemically implausible assignment  &    253 &    83 &    81 &    417 \\
\midrule
Alloys                             &  7,926 & 2,691 & 2,704 & 13,321 \\
Labelled non-alloys                & 17,610 & 5,804 & 5,841 & 29,255 \\
\midrule
Total                              & 25,536 & 8,495 & 8,545 & 42,576 \\
\bottomrule
\end{tabular}
\end{table}

\begin{table}[t]
\centering
\small
\caption{Architecture and training settings for BERTOS and CrystaliteOS.}
\vspace{0.5\baselineskip}
\label{tab:os-settings}

\begin{tabularx}{\linewidth}{@{}l
    >{\centering\arraybackslash}X
    >{\centering\arraybackslash}X@{}}

\multicolumn{3}{c}{\textbf{(a)} Architecture} \\[4pt]
\toprule
Setting & \textbf{BERTOS} & \textbf{CrystaliteOS} \\
\midrule
Hidden dimension
    & 120 & 128 \\
Transformer layers
    & 12 & 12 \\
Attention heads
    & 4 & 4 \\
MLP hidden dimension
    & 512 & 256 \\
Dropout
    & 0.1 & 0.1 \\
Attention dropout
    & 0.1 & 0.1 \\
Coordinate embedding
    & -- & RFF, 32 freqs. \\
GEM distance bias
    & -- & Enabled \\
GEM edge-aware bias
    & -- & Enabled \\
Edge-bias freqs.\ / hidden / RBF / max
    & -- & 12 / 256 / 32 / 4 \\
GEM sharing
    & -- & Enabled \\
PBC search radius $R$
    & -- & 1 \\
\bottomrule

\addlinespace[5pt]
\multicolumn{3}{c}{\textbf{(b)} Training} \\[4pt]
\toprule
Setting & \textbf{BERTOS} & \textbf{CrystaliteOS} \\
\midrule
Optimiser
    & AdamW & AdamW \\
Learning rate
    & $10^{-3}$ & $10^{-3}$ \\
Weight decay
    & 0 & 0 \\
Batch size
    & 256 & 64 \\
Training epochs
    & 150 & 150 \\
LR schedule
    & Linear decay & Cosine decay \\
Warmup
    & 5\% of steps & 1000 steps \\
Gradient clipping
    & 1.0 & 1.0 \\
Precision
    & FP32 & FP32 \\
Checkpoint selection
    & Validation crystal accuracy
    & Validation crystal accuracy \\
MLM epochs
    & 5 epochs & -- \\
MLM mask probability
    & 0.15 & -- \\
MLM learning rate
    & $10^{-4}$ & -- \\
MLM LR schedule
    & Constant & -- \\
MLM weight decay
    & 0.01 & -- \\
\bottomrule

\end{tabularx}
\end{table}

\section{OxiGen Details}
\label{app:oxigen}

OxiGen adapts the MatterGen crystal diffusion model \citep{zeni_generative_2025}. A crystal $\boldsymbol{M}=(\boldsymbol{S},\boldsymbol{F},\boldsymbol{L})$ is generated by jointly denoising its categorical species assignment $\boldsymbol{S}$, fractional coordinates $\boldsymbol{F}$, and lattice matrix $\boldsymbol{L}$. The forward diffusion process factorises across the three components,
\begin{equation}
q(\boldsymbol{M}_t\mid\boldsymbol{M}_r)=q(\boldsymbol{S}_t\mid\boldsymbol{S}_r)\,q(\boldsymbol{F}_t\mid\boldsymbol{F}_r)\,q(\boldsymbol{L}_t\mid\boldsymbol{L}_r)\qquad 0\leq r<t\leq 1.
\end{equation}
We retain MatterGen's coordinate diffusion, lattice diffusion and network architecture. We briefly summarise these components below and describe our modified species diffusion in Appendix~\ref{app:diffusion}.

\noindent\textbf{Coordinate diffusion.}\quad Coordinate diffusion is performed on fractional coordinates, which naturally live on the flat torus $\mathbb{T}^3=\mathbb{R}^3/\mathbb{Z}^3$. The diffusion process uses variance-exploding diffusion with wrapped-normal noise. For a fractional coordinate $\boldsymbol{f}$, the noising process is defined as
\begin{equation}
q(\boldsymbol{f}_t\mid \boldsymbol{f}_0)
=
\mathcal{N}_{\mathrm W}
\!\left(
\boldsymbol{f}_t;\boldsymbol{f}_0,
\sigma_t^2(N)\boldsymbol{I}
\right),
\end{equation}
where $\mathcal{N}_{\mathrm W}$ denotes a Gaussian distribution wrapped onto $[0,1)^3$. As $\sigma_t$ increases, this distribution approaches the uniform distribution on the unit torus. To approximately remove the dependence of the corresponding Cartesian noise scale on unit-cell size, the fractional-coordinate standard deviation is rescaled according to
\begin{equation}
\sigma_t(N)=\frac{\sigma_t}{N^{1/3}},
\end{equation}
where $N$ is the number of atomic sites in the unit cell. Since Cartesian and fractional coordinates satisfy $\boldsymbol{X}=\boldsymbol{L}\boldsymbol{F}$, the network predicts coordinate scores in Cartesian space and then transforms these to fractional-coordinate scores.

\noindent\textbf{Lattice diffusion.}\quad The arbitrary global rotation of the unit cell is removed using the polar decomposition
\begin{equation}
\boldsymbol{L}=\boldsymbol{U}\tilde{\boldsymbol{L}},
\end{equation}
where $\boldsymbol{U}$ is orthogonal and $\tilde{\boldsymbol{L}}$ is symmetric positive definite. Diffusion is then performed on $\tilde{\boldsymbol{L}}$ and restricted to the space of symmetric matrices. A variance-preserving diffusion process is used to avoid unrealistically large noisy unit cells. The noising process is defined as
\begin{equation}
q(\tilde{\boldsymbol{L}}_t\mid\tilde{\boldsymbol{L}}_0)
=
\mathcal{N}\!\left(
\sqrt{\bar{\alpha}_t}\tilde{\boldsymbol{L}}_0
+
(1-\sqrt{\bar{\alpha}_t})\mu(N)\boldsymbol{I},
\,
(1-\bar{\alpha}_t)\sigma_L^2(N)\boldsymbol{I}
\right),
\end{equation}
where $\bar{\alpha}_t$ is the variance-preserving signal schedule. The prior is centred on an approximately cubic lattice, and $\mu(N)$ and $\sigma_L(N)$ are scaled with the number of atoms so that the typical atomic density and signal-to-noise ratio remain approximately independent of system size.

\noindent\textbf{Network architecture.}\quad The reverse process is parameterised by an SE(3)-equivariant GemNet-dT message-passing network \citep{gasteiger_gemnet_2021}, which processes the complete noisy crystal $\boldsymbol{M}_t$ together with the diffusion time. The network jointly predicts Cartesian coordinate scores, symmetric lattice scores, and logits for the clean species assignments. Thus, although the three forward corruption processes factorise, their reverse predictions are coupled through a single network conditioned on the full noisy crystal.

\subsection{Species Diffusion}
\label{app:diffusion}

MatterGen uses an absorbing-state D3PM \citep{austin_structured_2021} discrete diffusion process over atomic identities $\boldsymbol{A}\in\mathcal{A}^N$. We instead characterise the categorical state at each site as a species $s=(a,z)$ consisting of an element $a$ and oxidation state $z$, and replace the discrete-time D3PM with the continuous-time absorbing-state diffusion process of MDLM \citep{sahoo_simple_2024}. 

\noindent\textbf{Species vocabulary.}\quad We use the species vocabulary defined in Appendix~\ref{app:os-vocab} and additionally include $z=0$ for elements classified as metals by \texttt{pymatgen}. The resulting vocabulary contains 428 species over 90 elements and a mask token, giving a total model vocabulary of 429 states. MatterGen initialises each atomic node with an element embedding $e_{\mathrm{element}}$; OxiGen replaces this with
\begin{equation}
e
=
e_{\mathrm{element}}
+
\omega_{\mathrm{os}} e_{\mathrm{os}},
\label{eq:os-embedding}
\end{equation}
where $e_{\mathrm{os}}$ is a learned oxidation-state embedding of the same dimension and $\omega_{\mathrm{os}}$ is a learned global scalar initialised at one. The categorical output remains a single linear head, now defined over the joint element--oxidation-state species vocabulary.

\noindent\textbf{Forward diffusion.}\quad Let $m$ denote the absorbing mask token and let $\alpha_t\in[0,1]$ denote the probability that a site remains unmasked at time $t$, with $\alpha_t$ monotonically decreasing from $\alpha_0=1$ to $\alpha_1=0$ in the idealised continuous-time process. Given a clean species assignment $\boldsymbol{S}_0=(s_{0,1},\ldots,s_{0,N})$, sites are masked independently according to
\begin{equation}
q(\boldsymbol{S}_t\mid\boldsymbol{S}_0)
=
\prod_{i=1}^N
\left[
\alpha_t\,\ind[s_{t,i}=s_{0,i}]
+
(1-\alpha_t)\,\ind[s_{t,i}=m]
\right].
\label{eq:mdlm-forward}
\end{equation}
For $0\leq r<t\leq 1$, the corresponding single-site posterior is
\begin{equation}
q(s_{r,i}\mid s_{t,i},s_{0,i})
=
\begin{cases}
\ind[s_{r,i}=s_{t,i}],
& s_{t,i}\neq m, \\[4pt]
\displaystyle
\frac{1-\alpha_r}{1-\alpha_t}\ind[s_{r,i}=m]
+
\frac{\alpha_r-\alpha_t}{1-\alpha_t}\ind[s_{r,i}=s_{0,i}],
& s_{t,i}=m.
\end{cases}
\label{eq:mdlm-posterior}
\end{equation}

\noindent\textbf{Denoising objective.}\quad The denoiser predicts $p_{\theta,i}(s_{0,i}\mid\boldsymbol{M}_t,t)$ from the complete noisy crystal $\boldsymbol{M}_t=(\boldsymbol{S}_t,\boldsymbol{F}_t,\boldsymbol{L}_t)$. Following \citet{sahoo_simple_2024}, the denoiser assigns zero probability to the mask token and deterministically copies the species at sites that are already unmasked. Consequently, only masked sites contribute to the denoising objective. For the factorised clean-state prediction
\begin{equation}
p_\theta(\boldsymbol{S}_0\mid\boldsymbol{M}_t,t)
=
\prod_{i=1}^N p_{\theta,i}(s_{0,i}\mid\boldsymbol{M}_t,t),
\end{equation}
the continuous-time MDLM NELBO for the species component is
\begin{equation}
\mathcal{L}_{\mathrm{MDLM}}
=
\mathbb{E}_{\boldsymbol{M}_0}
\int_0^1
\mathbb{E}_{\boldsymbol{M}_t\mid\boldsymbol{M}_0}
\left[
\lambda(t)
\sum_{i=1}^N
\ind[s_{t,i}=m]\,
\left(
-\log p_{\theta,i}(s_{0,i}\mid\boldsymbol{M}_t,t)
\right)
\right]
\,dt,
\label{eq:mdlm}
\end{equation}
where
\begin{equation}
\lambda(t)=-\frac{\alpha_t'}{1-\alpha_t}.
\label{eq:mdlm-weight}
\end{equation}
In implementation, we use the numerically stabilised schedule $\alpha_t=1-(1-\epsilon)t$ with $\epsilon=10^{-3}$, which approximates the fully masked terminal condition $\alpha_1=0$ of the idealised continuous-time process. 

\noindent\textbf{Structured training objective.}\quad In the ablation study, we additionally train OxiGen using the charge-neutral clean-state distribution $\widetilde p_\theta(\boldsymbol{S}_0\mid\boldsymbol{M}_t,t)$ defined in Eq.~\eqref{eq:structured}. A general derivation of the continuous-time masked-diffusion NELBO for arbitrary non-factorised clean-state distributions is given by \citet{van_krieken_neurosymbolic_2025}. Their Theorem~C.1 shows that, under zero masking probability, carry-over unmasking, and a proper fully masked terminal prior, the continuous-time objective depends only on the one-site marginals of the clean-state distribution.

Consider a reverse interval $[r,t]$ of width $\Delta t=t-r$. For a site that is masked at time $t$, Eq.~\eqref{eq:mdlm-posterior} gives the probability of revealing its clean species during this interval as
\begin{equation}
\frac{\alpha_r-\alpha_t}{1-\alpha_t}
=
-\frac{\alpha_t'}{1-\alpha_t}\Delta t
+
O(\Delta t^2)
=
\lambda(t)\Delta t+O(\Delta t^2).
\label{eq:unmask-rate}
\end{equation}
Thus, revealing one site has probability $O(\Delta t)$, whereas revealing two or more sites in the same interval has probability $O(\Delta t^2)$ or smaller. Summed over $O(\Delta t^{-1})$ reverse intervals, the contribution of multi-site reveal events vanishes as $\Delta t\rightarrow 0$. Consequently, only single-site reveal events contribute in the continuous-time limit, and the NELBO depends on the structured distribution only through its one-site marginals $\widetilde p_{\theta,i}(s_{0,i}\mid\boldsymbol{M}_t,t)$.

Replacing the factorised clean-state probabilities in Eq.~\eqref{eq:mdlm} by these exact structured marginals gives
\begin{equation}
\mathcal{L}_{\mathrm{neutral}}
=
\mathbb{E}_{\boldsymbol{M}_0}
\int_0^1
\mathbb{E}_{\boldsymbol{M}_t\mid\boldsymbol{M}_0}
\left[
\lambda(t)
\sum_{i=1}^N
\ind[s_{t,i}=m]\,
\left(
-\log
\widetilde p_{\theta,i}(s_{0,i}\mid\boldsymbol{M}_t,t)
\right)
\right]
\,dt.
\label{eq:structured-loss}
\end{equation}
The marginals $\widetilde p_{\theta,i}$ are computed exactly using the forward--backward procedure in Section~\ref{sec:inference}. The structured objective therefore has the same site-wise form as the standard MDLM objective, but evaluates each masked-site loss under the charge-neutral clean-state distribution. We use this objective only for the ablation study; the default OxiGen model is trained with Eq.~\eqref{eq:mdlm} and applies the structured output layer only during sampling.

\subsection{Training and sampling}
\label{app:oxigen-train}

For the coordinates and lattice, we follow the training and sampling procedures of MatterGen. Species are generated using the charge-neutral sampling procedure described in Section~\ref{sec:inference}. Following MatterGen, generation excludes the noble gases, Tc, Pm, and all elements with atomic number $Z\geq84$. We retain MatterGen's network, coordinate-diffusion, and lattice-diffusion training and sampling hyperparameters without additional tuning. The architecture, optimisation, diffusion, and sampling settings are summarised in Table~\ref{tab:oxigen}.

\subsection{Computational Overhead}
\label{app:oxigen-compute}

Structured inference requires a single backward dynamic-programming pass to draw a charge-neutral sample at each diffusion step using Eq.~\eqref{eq:sampler}, while structured training requires one forward and one backward dynamic-programming pass to compute the marginals of the structured distribution using Eq.~\eqref{eq:marginal}. Each pass has time complexity $O(N^2|\mathcal{S}|(z_{\max}-z_{\min}))$, where $z_{\max}$ and $z_{\min}$ are the maximum and minimum oxidation states in the species vocabulary.

Table~\ref{tab:compute-cost} reports the overhead of the structured output layer relative to the same model without it. Training time is measured per epoch on MP-20-OS, and sampling time per batch of 128 structures over a full 1000-step reverse-diffusion trajectory, both on a single NVIDIA GH200 120GB with batch size 128. Training OxiGen for 900 epochs on four GH200 GPUs with batch size 128 per GPU (512 in total; Table~\ref{tab:oxigen}) takes approximately 3.5 hours.

\begin{table}[h]
\centering
\small
\setlength{\tabcolsep}{7.0pt}
\renewcommand{\arraystretch}{1.08}
\caption{Computational overhead of the structured output layer, reported as mean $\pm$ std over five epochs or five batches. Percentage changes are relative to the model without the structured output layer.}
\vspace{0.5\baselineskip}
\label{tab:compute-cost}
\begin{tabular}{@{}ccc@{}}
    \toprule
    \multirow{2}{*}{Structured Output Layer}
    & Training Time
    & Sampling Time \\
    & (s/epoch) $\downarrow$
    & (s/batch) $\downarrow$ \\
    \midrule
    \xmark
        & $78.797\msmall{\pm 0.101}$
        & $309.527\msmall{\pm 4.923}$ \\
    \cmark
        & $98.906\msmall{\pm 0.057}$
        & $329.716\msmall{\pm 4.809}$ \\
    \midrule
    Change (\%)
        & $+25.52$
        & $+6.52$ \\
    \bottomrule
\end{tabular}
\vspace{-1\baselineskip}
\end{table}

\begin{table}[t]
\centering
\small
\caption{Architecture, training, and sampling settings for OxiGen.}
\vspace{0.5\baselineskip}
\label{tab:oxigen}
\begin{tabular}{@{}ll@{}}
\toprule
Setting & Value \\
\midrule
Hidden / atom / edge dimension & 512 / 512 / 512 \\
Interaction blocks & 4 \\
Cutoff & $7.0$~\AA \\
Maximum neighbours & 50 \\
\midrule
Optimiser & Adam \\
Learning rate & $10^{-4}$ \\
LR schedule & ReduceLROnPlateau \\
LR factor / patience / minimum & 0.6 / 100 epochs / $10^{-6}$ \\
Batch size & 512 \\
Training epochs & 900 \\
Precision & FP32 \\
Gradient clipping & 0.5, value-based \\
Checkpoint selection & Last epoch \\
EMA & None \\
\midrule
Cell / position / species loss weights & $1.0$ / $0.1$ / $0.01$ \\
Mask schedule & $\alpha_t=1-(1-\epsilon)t$ \\
$\epsilon$ / training-time range & $10^{-3}$ / $t\sim\mathcal{U}(10^{-3},1)$ \\
Reverse steps & 1000 \\
Reverse-time grid & Uniform in $t$ \\
Position/cell corrector & 1 Langevin step \\
Species corrector & None \\
\bottomrule
\end{tabular}
\end{table}

\section{Band-Gap Conditional Generation}
\label{app:band-gap}

\noindent\textbf{Conditional generation.}\quad We fine-tune MatterGen and OxiGen for band-gap conditional generation following \citet{zeni_generative_2025}. Both models are fine-tuned on MP-20-OS using the DFT band-gap labels provided with MP-20. Band-gap values are standardised and encoded using a fixed sinusoidal embedding, which is incorporated before each GemNet-dT message-passing layer through an adapter consisting of a two-layer MLP and a zero-initialised linear mix-in layer. Fine-tuning updates both the pre-trained denoiser and the conditioning parameters. The models are trained with conditional dropout and sampled using classifier-free guidance, applied to both the continuous coordinate and lattice scores and the discrete atomic-identity or species predictions. Fine-tuning and generation settings are summarised in Table~\ref{tab:band-gap}.

\noindent\textbf{Band-gap prediction.}\quad We evaluate conditional generation using iComFormer \citep{yan_complete_2024}, an SE(3)-invariant crystal transformer for property prediction. Following \citet{yan_complete_2024}, we train the model on the Materials Project 2018.6.1 band-gap benchmark using the authors' provided data split, code, and hyperparameters. Our trained model achieves a test mean absolute error of $0.193~\mathrm{eV}$, matching the value reported by \citet{yan_complete_2024}. For generated crystals, band gaps are predicted from the MatterSim-relaxed structures rather than the unrelaxed model outputs.

\begin{table}[t]
\centering
\small
\caption{Band-gap conditional fine-tuning and sampling settings for OxiGen and MatterGen.}
\vspace{0.5\baselineskip}
\label{tab:band-gap}
\begin{tabular}{@{}ll@{}}
\toprule
Setting & Value \\
\midrule
Property embedding dimension & 512 \\
Null condition & Learned embedding \\
Fine-tuning mode & Full-model \\
Initialisation & Unconditional last epoch \\
Optimiser & Adam \\
Learning rate & $5\times10^{-6}$ \\
LR schedule & ReduceLROnPlateau \\
LR factor / patience / minimum & 0.6 / 100 epochs / $10^{-6}$ \\
Batch size & 128 \\
Epochs & 200 \\
Conditional dropout & 0.2 \\
Gradient clipping / precision & 0.5 / FP32 \\
\midrule
Guidance strength $\gamma$ & 2 \\
Reverse steps & 1000 \\
\bottomrule
\end{tabular}
\end{table}

\section{Evaluation Procedure}
\label{app:os-eval}

\subsection{Stability, uniqueness, and novelty}
\label{app:sun-eval}

We follow the released MatterGen evaluation procedure \citep{zeni_generative_2025}. Evaluation excludes structures containing noble gases, Tc, Pm, and all elements with atomic number $Z\geq84$ for all methods. Generated structures are first relaxed with MatterSim v1.0.0-1M \citep{yang_mattersim_2024}, and stability, uniqueness, novelty, and relaxation metrics are computed from the relaxed structures. Energy above the convex hull, $E_{\mathrm{hull}}$, is computed using \texttt{pymatgen} \citep{ong_python_2013} phase diagrams with the Alex-MP reference dataset and Materials Project 2020 compatibility corrections provided by \citet{zeni_generative_2025}. A structure is considered stable if $E_{\mathrm{hull}}\leq0.1~\mathrm{eV/atom}$. Uniqueness and novelty are computed via the \texttt{StructureMatcher} from \texttt{pymatgen} with $\texttt{ltol}=0.2$, $\texttt{stol}=0.3$, and $\texttt{angle\_tol}=5^\circ$. Uniqueness is evaluated within each generated set, and a structure is considered novel if it does not match any entry in the Alex-MP reference dataset. Oxidation-state labels are removed before structure matching. The relaxation RMSD is computed by matching each generated structure to its relaxed counterpart using \texttt{StructureMatcher} with $\texttt{ltol}=0.5$, $\texttt{stol}=0.5$, and $\texttt{angle\_tol}=10^\circ$. A structure is counted as S.U.N. if it is simultaneously stable, unique, and novel. Stability and S.U.N. rates are reported relative to the total number of generated structures. Uniqueness, novelty, mean $E_{\mathrm{hull}}$, and relaxation RMSD are computed over structures that were successfully relaxed.

\subsection{Oxidation-state metrics}
\label{app:os-metrics}

We evaluate compositional validity, the oxidation-state frequency score, and the oxidation-state distribution distance using \texttt{SMACT}~v4.0.0 \citep{davies_smact_2019}. We exclude elemental structures and structures for which all elements are classified as metals by \texttt{pymatgen}. Oxidation-state frequencies are taken from the ICSD-24 species-occurrence table with \texttt{consensus=3} and \texttt{include\_zero=False}. We compute compositional validity using \texttt{smact\_validity} with \texttt{consensus=3}, \texttt{commonality=low}, \texttt{include\_zero=False}, \texttt{mixed\_valence=False}, and \texttt{include\_alloys=False}.

For an element assignment $\boldsymbol{A}=(a_1,\ldots,a_N)$, let $\mathcal{N}(\boldsymbol{A})$ denote the set of valid charge-neutral oxidation-state assignments enumerated by \texttt{SMACT}. Because mixed valence is not permitted, each element is assigned a single oxidation state within any $\boldsymbol{S}\in\mathcal{N}(\boldsymbol{A})$. We write $\operatorname{uniq}(\boldsymbol{S})$ for the set of distinct species $(a,z)$ appearing in $\boldsymbol{S}$, and let $f_a(z)\in[0,1]$ denote the normalised ICSD-24 occurrence frequency of oxidation state $z$ for element $a$.

For each composition, we define the oxidation-state frequency score as the highest geometric mean species frequency over all valid assignments,
\begin{equation}
f_{\mathrm{OS}}(\boldsymbol{A})
=
100
\max_{\boldsymbol{S}\in\mathcal{N}(\boldsymbol{A})}
\left(
\prod_{(a,z)\in\operatorname{uniq}(\boldsymbol{S})}
f_a(z)
\right)^{1/|\operatorname{uniq}(\boldsymbol{S})|}.
\label{eq:os-frequency}
\end{equation}
Each distinct species contributes once, regardless of its multiplicity in the composition. We report the mean of $f_{\mathrm{OS}}(\boldsymbol{A})$ over structures with at least one valid assignment.

For Figure~\ref{fig:os-fidelity}, we additionally consider the minimum species frequency within each valid assignment. We define
\begin{equation}
f_{\mathrm{OS}}^{\mathrm{min}}(\boldsymbol{A})
=
\max_{\boldsymbol{S}\in\mathcal{N}(\boldsymbol{A})}
\min_{(a,z)\in\operatorname{uniq}(\boldsymbol{S})}
f_a(z),
\label{eq:min-frequency}
\end{equation}
which is the highest minimum species frequency achievable by any valid assignment for the composition. At a threshold $\tau\in[0,1]$, a composition is therefore counted as valid if it admits at least one assignment in which every constituent species has frequency at least $\tau$, equivalently if $f_{\mathrm{OS}}^{\mathrm{min}}(\boldsymbol{A})\geq\tau$. We report the fraction of evaluated structures satisfying this condition.

To compare oxidation-state distributions, we account for compositions that admit multiple valid assignments by assigning each $\boldsymbol{S}\in\mathcal{N}(\boldsymbol{A})$ a weight proportional to the product of its species frequencies,
\begin{equation}
q_{\boldsymbol{A}}(\boldsymbol{S})
=
\frac{
\displaystyle
\prod_{(a,z)\in\operatorname{uniq}(\boldsymbol{S})}
f_a(z)
}{
\displaystyle
\sum_{\boldsymbol{S}'\in\mathcal{N}(\boldsymbol{A})}
\prod_{(a,z)\in\operatorname{uniq}(\boldsymbol{S}')}
f_a(z)
}.
\label{eq:assignment-weight}
\end{equation}
These weights define a distribution over the valid assignments of each composition. For a dataset $\mathcal{D}$, the oxidation-state distribution for element $a$ is then obtained by marginalising over these assignments,
\begin{equation}
p_{\mathcal{D}}(z\mid a)
\propto
\sum_{\boldsymbol{A}\in\mathcal{D}}
\sum_{\boldsymbol{S}\in\mathcal{N}(\boldsymbol{A})}
q_{\boldsymbol{A}}(\boldsymbol{S})
\ind[(a,z)\in\operatorname{uniq}(\boldsymbol{S})],
\label{eq:os-distribution}
\end{equation}
with the distribution normalised over $z$ separately for each element.

Finally, we compare the generated and reference oxidation-state distributions using the Jensen--Shannon distance, macro-averaged over the set $\mathcal{A}'$ of elements present in both datasets,
\begin{equation}
d_{\mathrm{OS}}
=
\frac{1}{|\mathcal{A}'|}
\sum_{a\in\mathcal{A}'}
\sqrt{
\operatorname{JSD}
\left(
p_{\mathrm{gen}}(\cdot\mid a)
\,\Vert\,
p_{\mathrm{ref}}(\cdot\mid a)
\right)
}.
\label{eq:os-distance}
\end{equation}
We use the MP-20-OS test split as the reference distribution. Compositions with no valid assignment are excluded from the oxidation-state frequency and distribution metrics.

\section{Additional Results}

\subsection{De Novo Generation}

Table~\ref{tab:app-de-novo} reports the stability, uniqueness, novelty, and relaxation metrics for the \textit{de novo} generation experiments in Section~\ref{sec:de-novo}. The evaluated methods show the expected trade-off between stability and novelty. Crystalite achieves the highest stable fraction but the lowest novelty among the compared methods, while MatterGen achieves the highest novelty. OxiGen achieves a strong balance across these metrics, resulting in the highest S.U.N. rate, together with the lowest relaxation RMSD.

\noindent\textbf{Proxy metrics.}\quad Table~\ref{tab:app-proxy} reports the proxy metrics proposed by \citet{xie_crystal_2022}. All metrics are computed on the unrelaxed structures. Validity and coverage are reported over all generated structures, while the distribution metrics are computed over the subset of structures that are both compositionally and structurally valid. Compositional validity here is evaluated over all structures, including alloys and elemental structures, with \texttt{include\_alloys=True}, and therefore differs from the results reported in Table~\ref{tab:de-novo}. OxiGen is comparable to MatterGen on structural validity and coverage, and slightly worse on element-count distribution distance.

\noindent\textbf{LeMat-GenBench.}\quad Table~\ref{tab:leaderboard} reports results under the standardised LeMat-GenBench evaluation procedure \citep{betala_lematgenbench_2026}. We evaluate 2,500 structures generated by the OxiGen seed with the highest S.U.N. rate under our primary evaluation, pre-relaxed using MatterSim. We compare against methods trained on MP-20 on the public LeMat-GenBench leaderboard, with results taken directly from the leaderboard; OxiGen is instead trained on MP-20-OS, a filtered subset of MP-20 (Appendix~\ref{app:mp20-labelling}). The baselines evaluated on pre-relaxed structures are WyFormer and WyFormer-DFT \citep{kazeev_wyckoff_2025}, PLaID++ \citep{xu_plaid_2026}, OMatG-FC \citep{martirossyan_framework_2026}, MatterGen \citep{zeni_generative_2025}, MiAD \citep{okhotin_miad_2025}, OMatG \citep{hollmer_open_2025}, MCFlow \citep{seong_multimodal_2026}, and Crystalite \citep{veljkovic_crystalite_2026}. The remaining baselines were evaluated on unrelaxed structures: Crystal-GFN \citep{ai4science_crystal-gfn_2023}, LLaMat2 and LLaMat3 \citep{ahlawat_family_2026}, ADiT \citep{joshi_all-atom_2025}, SymmCD \citep{levy_symmcd_2024}, CrystalFormer \citep{cao_space_2025}, CrystaLLM-$\pi$ \citep{bone_discovery_2025}, DiffCSP++ \citep{jiao_space_2024}, DiffCSP \citep{jiao_crystal_2023}, Chemeleon1 \citep{park_exploration_2025}, and Chemeleon2 \citep{park_guiding_2026}. LeMat-GenBench reports the fractions of stable ($E_{\mathrm{hull}}\leq0$) and metastable ($E_{\mathrm{hull}}\leq0.1$~eV/atom) structures; S.U.N. denotes the fraction of generated structures that are stable, unique, and novel, whereas M.S.U.N. uses metastability in place of stability. These definitions differ from our primary evaluation, which uses $E_{\mathrm{hull}}\leq0.1$~eV/atom as the stability threshold. Under the LeMat-GenBench procedure, OxiGen achieves the highest validity, uniqueness, and M.S.U.N. rate among all evaluated methods, and is competitive across the other metrics.

\begin{table}[t]
    \centering
    \small
    \setlength{\tabcolsep}{4.2pt}
    \renewcommand{\arraystretch}{1.10}
    \caption{Additional \textit{de novo} generation results, reported as mean $\pm$ std over five independently seeded training runs, with 10,240 generated structures per seed. Best values are shown in \textbf{bold} and second-best values are \underline{underlined}.}
    \vspace{0.5\baselineskip}
    \label{tab:app-de-novo}

    \begin{tabular}{@{}lcccccc@{}}
        \toprule
        \multirow{2}{*}{Method}
        & Stable
        & Unique
        & Novel
        & S.U.N.
        & $E_{\mathrm{hull}}$
        & RMSD \\

        & (\%) $\uparrow$
        & (\%) $\uparrow$
        & (\%) $\uparrow$
        & (\%) $\uparrow$
        & (eV/atom) $\downarrow$
        & (\AA) $\downarrow$ \\
        \midrule

        DiffCSP
            & $43.34\msmall{\pm 1.29}$
            & $95.80\msmall{\pm 0.58}$
            & $71.58\msmall{\pm 1.57}$
            & $18.24\msmall{\pm 1.38}$
            & $0.2007\msmall{\pm 0.0101}$
            & $0.3555\msmall{\pm 0.0427}$ \\

        MatterGen
            & $46.11\msmall{\pm 3.17}$
            & $\underline{97.74}\msmall{\pm 0.49}$
            & $\mathbf{75.32}\msmall{\pm 2.72}$
            & $\underline{24.54}\msmall{\pm 0.91}$
            & $0.1799\msmall{\pm 0.0085}$
            & $\underline{0.1747}\msmall{\pm 0.0203}$ \\

        FlowMM
            & $36.04\msmall{\pm 0.83}$
            & $\mathbf{97.99}\msmall{\pm 0.21}$
            & $\underline{74.07}\msmall{\pm 1.01}$
            & $13.85\msmall{\pm 0.39}$
            & $0.2481\msmall{\pm 0.0104}$
            & $0.6704\msmall{\pm 0.0347}$ \\

        OMatG
            & $49.12\msmall{\pm 0.39}$
            & $95.12\msmall{\pm 0.33}$
            & $65.42\msmall{\pm 0.87}$
            & $19.03\msmall{\pm 0.78}$
            & $0.1831\msmall{\pm 0.0030}$
            & $0.5244\msmall{\pm 0.0278}$ \\

        Crystalite
            & $\mathbf{67.17}\msmall{\pm 2.18}$
            & $93.03\msmall{\pm 0.91}$
            & $52.25\msmall{\pm 2.61}$
            & $22.83\msmall{\pm 0.76}$
            & $\mathbf{0.1437}\msmall{\pm 0.0083}$
            & $0.2496\msmall{\pm 0.0198}$ \\

        \midrule

        \textbf{OxiGen}
            & $\underline{52.59}\msmall{\pm 2.25}$
            & $96.45\msmall{\pm 0.64}$
            & $70.68\msmall{\pm 2.31}$
            & $\mathbf{26.38}\msmall{\pm 1.01}$
            & $\underline{0.1575}\msmall{\pm 0.0080}$
            & $\mathbf{0.1679}\msmall{\pm 0.0117}$ \\

        \bottomrule
    \end{tabular}
\end{table}

\begin{table}[t]
    \centering
    \small
    \setlength{\tabcolsep}{5.0pt}
    \renewcommand{\arraystretch}{1.10}
    \caption{Proxy metrics for \textit{de novo} generation, reported as mean $\pm$ std over five independently seeded training runs, with 10,240 generated structures per seed. Best values are shown in \textbf{bold} and second-best values are \underline{underlined}.}
    \vspace{0.5\baselineskip}
    \label{tab:app-proxy}

    \begin{tabular}{@{}lcccccc@{}}
        \toprule
        \multirow{2}{*}{Method}
        & \multicolumn{2}{c}{Validity}
        & \multicolumn{2}{c}{Coverage}
        & \multicolumn{2}{c}{Distribution} \\
        \cmidrule(lr){2-3}
        \cmidrule(lr){4-5}
        \cmidrule(lr){6-7}

        & Comp. (\%) $\uparrow$
        & Struct. (\%) $\uparrow$
        & Cov-R (\%) $\uparrow$
        & Cov-P (\%) $\uparrow$
        & wdist $\rho$ $\downarrow$
        & wdist $N_{\mathrm{el}}$ $\downarrow$ \\
        \midrule

        DiffCSP
            & $87.29\msmall{\pm 1.97}$
            & $99.89\msmall{\pm 0.05}$
            & $99.58\msmall{\pm 0.05}$
            & $99.70\msmall{\pm 0.10}$
            & $0.2514\msmall{\pm 0.1340}$
            & $0.5052\msmall{\pm 0.0557}$ \\

        MatterGen
            & $88.71\msmall{\pm 0.64}$
            & $\underline{99.97}\msmall{\pm 0.04}$
            & $\mathbf{99.71}\msmall{\pm 0.08}$
            & $99.69\msmall{\pm 0.07}$
            & $\mathbf{0.1416}\msmall{\pm 0.0576}$
            & $0.2398\msmall{\pm 0.0791}$ \\

        FlowMM
            & $87.58\msmall{\pm 0.17}$
            & $91.80\msmall{\pm 0.86}$
            & $99.27\msmall{\pm 0.11}$
            & $98.83\msmall{\pm 0.18}$
            & $1.3027\msmall{\pm 0.4222}$
            & $0.1200\msmall{\pm 0.0125}$ \\

        OMatG
            & $89.06\msmall{\pm 0.37}$
            & $99.11\msmall{\pm 0.19}$
            & $99.59\msmall{\pm 0.09}$
            & $\underline{99.83}\msmall{\pm 0.03}$
            & $0.2672\msmall{\pm 0.0152}$
            & $\mathbf{0.0549}\msmall{\pm 0.0097}$ \\

        Crystalite
            & $\underline{89.12}\msmall{\pm 0.41}$
            & $99.68\msmall{\pm 0.06}$
            & $\underline{99.67}\msmall{\pm 0.09}$
            & $\mathbf{99.87}\msmall{\pm 0.04}$
            & $\underline{0.1541}\msmall{\pm 0.0266}$
            & $\underline{0.1151}\msmall{\pm 0.0190}$ \\

        \midrule

        \textbf{OxiGen}
            & $\mathbf{95.24}\msmall{\pm 0.44}$
            & $\mathbf{99.98}\msmall{\pm 0.01}$
            & $99.60\msmall{\pm 0.08}$
            & $99.80\msmall{\pm 0.01}$
            & $0.2045\msmall{\pm 0.0869}$
            & $0.3137\msmall{\pm 0.0340}$ \\

        \bottomrule
    \end{tabular}
\end{table}

\subsection{Ablation Study}

In this section we provide additional results for the ablations in Section~\ref{sec:de-novo}.

\noindent\textbf{Stability, uniqueness, and novelty.}\quad Table~\ref{tab:app-ablation-stable} reports the stability, uniqueness, novelty, and relaxation metrics. Relative to the element-representation baseline, OxiGen slightly increases both the stability and S.U.N. rate. Comparing the species models, structured sampling increases stability and slightly increases the S.U.N. rate, while modestly reducing uniqueness and novelty. Structured training provides no further improvement.

\noindent\textbf{Oxidation-state agreement.}\quad Table~\ref{tab:app-ablation-os} reports the oxidation-state agreement between OxiGen and CrystaliteOS. We label each unrelaxed generated structure using CrystaliteOS, excluding metallic alloys and elemental structures. Site agreement is the fraction of atomic sites for which the generated oxidation state matches the CrystaliteOS prediction and crystal agreement is the fraction of generated crystals for which the two assignments agree at every site. These metrics are not defined for the element representation because it does not explicitly generate oxidation states. Structured sampling substantially improves the agreement between the generated oxidation states and the CrystaliteOS assignments. Structured training again provides no additional improvement. These results suggest that structured sampling not only guarantees charge neutrality, but improves the proportion of generated oxidation states that agree with those inferred from the generated structure.

\begin{table}[h]
    \centering
    \small
    \renewcommand{\arraystretch}{1.10}
    \caption{
    \textit{De novo} generation results under the LeMat-GenBench evaluation procedure.
    OxiGen results are computed over 2,500 generated structures; results for all other models trained only on MP-20 are taken from the public LeMat-GenBench leaderboard.
    Rows are grouped according to whether the submitted structures are pre-relaxed.
    Best and second-best values within the pre-relaxed group are shown in \textbf{bold} and \underline{underlined}, respectively.
    }
    \vspace{0.5\baselineskip}
    \label{tab:leaderboard}

    \begin{tabular*}{\textwidth}{@{\extracolsep{\fill}}lcccccccc@{}}
        \toprule
        \multirow{2}{*}{Model}
        & Valid & Unique & Novel & Stable & Metastable & S.U.N. & M.S.U.N. & $E_{\mathrm{hull}}$ \\
        & (\%) $\uparrow$ & (\%) $\uparrow$ & (\%) $\uparrow$ & (\%) $\uparrow$
        & (\%) $\uparrow$ & (\%) $\uparrow$ & (\%) $\uparrow$ & (eV/atom) $\downarrow$ \\

        \midrule
        \multicolumn{9}{c}{\textbf{Pre-Relaxed Structures}} \\
        \midrule

        WyFormer        & 93.4 & 93.0 & 66.4 & 0.5 & 15.7 & 0.1 & 1.9 & 0.4988 \\
        PLaID++         & 96.0 & 77.8 & 24.2 & 12.4 & \underline{60.7} & 1.0 & 7.6 & 0.0854 \\
        WyFormer-DFT    & 95.2 & 95.0 & 66.4 & 3.7 & 24.8 & 0.4 & 7.8 & 0.2708 \\
        OMatG-FC        & \underline{97.2} & 92.8 & 28.9 & \textbf{18.4} & 57.9 & \textbf{1.7} & 12.0 & \textbf{0.0694} \\
        MatterGen       & 95.7 & 95.1 & \textbf{70.5} & 2.0 & 33.4 & 0.2 & 15.0 & 0.1834 \\
        MiAD            & 96.2 & 94.3 & 40.2 & 6.4 & \textbf{62.0} & 1.0 & 16.6 & \underline{0.0804} \\
        OMatG           & 96.4 & 95.2 & 51.2 & 11.6 & 49.8 & 1.0 & 18.0 & 0.0956 \\
        MCFlow          & \underline{97.2} & \underline{96.3} & 52.2 & 11.9 & 49.3 & 0.7 & 18.9 & 0.0987 \\
        Crystalite      & \underline{97.2} & 95.8 & 53.2 & \underline{12.7} & 51.6 & \underline{1.5} & \underline{22.6} & 0.0905 \\
        \textbf{OxiGen} & \textbf{99.6} & \textbf{98.5} & \underline{68.8} & 5.0 & 49.5 & 0.6 & \textbf{26.7} & 0.1203 \\

        \midrule
        \multicolumn{9}{c}{\textbf{Non-Pre-Relaxed Structures}} \\
        \midrule

        Crystal-GFN     & 51.7 & 51.7 & 51.7 & 0.0 & 0.0 & 0.0 & 0.0 & 2.0858 \\
        LLaMat3         & 15.4 & 15.2 & 10.5 & 0.1 & 2.1 & 0.0 & 0.2 & 1.7067 \\
        ADiT            & 90.6 & 87.8 & 26.0 & 0.4 & 36.5 & 0.0 & 1.0 & 0.3333 \\
        LLaMat2         & 84.4 & 81.4 & 30.0 & 0.7 & 34.7 & 0.1 & 2.1 & 0.4395 \\
        SymmCD          & 73.4 & 73.0 & 47.0 & 1.4 & 18.6 & 0.1 & 2.4 & 0.8761 \\
        CrystalFormer   & 69.9 & 69.4 & 31.8 & 1.4 & 28.8 & 0.0 & 3.1 & 0.7039 \\
        CrystaLLM-$\pi$ & 86.8 & 84.9 & 24.9 & 2.9 & 49.0 & 0.3 & 3.6 & 0.3205 \\
        DiffCSP++       & 95.3 & 95.1 & 62.0 & 1.0 & 26.4 & 0.2 & 5.0 & 0.4093 \\
        DiffCSP         & 95.7 & 94.8 & 66.2 & 2.3 & 29.8 & 0.1 & 8.5 & 0.2747 \\
        Chemeleon1      & 95.4 & 94.8 & 64.0 & 3.6 & 34.4 & 0.3 & 11.3 & 0.2055 \\
        Chemeleon2      & 95.2 & 88.1 & 71.6 & 0.0 & 39.8 & 0.0 & 21.2 & 0.1557 \\

        \bottomrule
    \end{tabular*}
\end{table}

\begin{table}[h]
\centering
\small
\setlength{\tabcolsep}{4.5pt}
\renewcommand{\arraystretch}{1.08}
\caption{
Ablation results for stability, uniqueness, novelty, and relaxation quality. Results are reported as mean $\pm$ std over five independently seeded training runs, with 10,240 generated structures per seed.
The method names indicate whether structured training / structured sampling are used, respectively.
Best values are shown in \textbf{bold} and second-best values are \underline{underlined}; the shaded row is OxiGen.
}
\vspace{0.5\baselineskip}
\label{tab:app-ablation-stable}

\begin{tabular}{@{}lcccccc@{}}
    \toprule
    Method
    & Stable
    & Unique
    & Novel
    & S.U.N.
    & $E_{\mathrm{hull}}$
    & RMSD \\

    &
    (\%) $\uparrow$
    & (\%) $\uparrow$
    & (\%) $\uparrow$
    & (\%) $\uparrow$
    & (eV/atom) $\downarrow$
    & (\AA) $\downarrow$ \\

    \midrule

    Element \xmark/\xmark
        & $47.47\msmall{\pm 2.65}$
        & $\mathbf{97.48}\msmall{\pm 0.36}$
        & $\mathbf{74.68}\msmall{\pm 2.03}$
        & $25.21\msmall{\pm 0.60}$
        & $0.1709\msmall{\pm 0.0103}$
        & $0.1728\msmall{\pm 0.0098}$ \\

    Species \xmark/\xmark
        & $49.72\msmall{\pm 2.53}$
        & $\underline{97.28}\msmall{\pm 0.35}$
        & $\underline{73.20}\msmall{\pm 1.80}$
        & $\underline{26.08}\msmall{\pm 0.92}$
        & $\underline{0.1646}\msmall{\pm 0.0035}$
        & $\mathbf{0.1616}\msmall{\pm 0.0102}$ \\

    \rowcolor{gray!15}
    Species \xmark/\cmark
        & $\mathbf{52.59}\msmall{\pm 2.25}$
        & $96.45\msmall{\pm 0.64}$
        & $70.68\msmall{\pm 2.31}$
        & $\mathbf{26.38}\msmall{\pm 1.01}$
        & $\mathbf{0.1575}\msmall{\pm 0.0080}$
        & $\underline{0.1679}\msmall{\pm 0.0117}$ \\

    Species \cmark/\cmark
        & $\underline{50.35}\msmall{\pm 2.39}$
        & $96.80\msmall{\pm 0.70}$
        & $72.04\msmall{\pm 2.03}$
        & $25.62\msmall{\pm 0.69}$
        & $0.1696\msmall{\pm 0.0075}$
        & $0.1853\msmall{\pm 0.0175}$ \\

    \bottomrule
\end{tabular}
\end{table}

\begin{table}[h]
\centering
\small
\setlength{\tabcolsep}{4.5pt}
\renewcommand{\arraystretch}{1.08}
\caption{
Ablation results for charge neutrality and oxidation-state agreement.
Results are reported as mean $\pm$ std over five independently seeded training runs, with 10,240 generated structures per seed. The  method names indicate whether structured training / structured sampling are used, respectively.
Best values are shown in \textbf{bold} and second-best values are \underline{underlined}; the shaded row is OxiGen.
}
\vspace{0.5\baselineskip}
\label{tab:app-ablation-os}

\begin{tabular}{@{}lccc@{}}
    \toprule
    Method
    & Charge Neutrality (\%) $\uparrow$
    & Site Agreement (\%) $\uparrow$
    & Crystal Agreement (\%) $\uparrow$ \\
    \midrule

    Element \xmark/\xmark
        & --
        & --
        & -- \\

    Species \xmark/\xmark
        & $\underline{39.08}\msmall{\pm 2.60}$
        & $74.45\msmall{\pm 1.84}$
        & $33.57\msmall{\pm 2.29}$ \\

    \rowcolor{gray!15}
    Species \xmark/\cmark
        & $\mathbf{100.00}\msmall{\pm 0.00}$
        & $\mathbf{90.69}\msmall{\pm 0.81}$
        & $\mathbf{72.47}\msmall{\pm 1.40}$ \\

    Species \cmark/\cmark
        & $\mathbf{100.00}\msmall{\pm 0.00}$
        & $\underline{90.45}\msmall{\pm 0.39}$
        & $\underline{71.49}\msmall{\pm 1.32}$ \\

    \bottomrule
\end{tabular}
\end{table}

\subsection{Band-Gap Conditional Generation}

Table~\ref{tab:app-bandgap-conditional} reports additional results for the band-gap-conditioned generation experiments in Section~\ref{sec:de-novo}. The band-gap RMSE measures the root-mean-square error between the conditioning target and the iComFormer-predicted band gap of the relaxed generated structures. OxiGen achieves a higher S.U.N. rate than MatterGen at all four conditioning targets, due to a substantially higher stability rate despite a lower novelty rate. Band-gap RMSE is comparable at 1~eV and better for OxiGen at the 3, 5, and 7~eV targets.

\begin{table}[H]
    \centering
    \small
    \setlength{\tabcolsep}{4.2pt}
    \renewcommand{\arraystretch}{1.10}
    \vspace{-1\baselineskip}
    \caption{
    Band-gap-conditioned generation results for target band gaps of 1, 3, 5, and 7~eV.
    Results are shown for a single fine-tuning seed, with 1,024 generated structures at each target.
    }
    \vspace{0.5\baselineskip}
    \label{tab:app-bandgap-conditional}

    \begin{tabular*}{\textwidth}{@{\extracolsep{\fill}}lcccccccc@{}}
        \toprule
        \multirow{2}{*}{Method}
        & Target
        & RMSE
        & Stable
        & Unique
        & Novel
        & S.U.N.
        & $E_{\mathrm{hull}}$
        & RMSD \\

        & (eV)
        & (eV) $\downarrow$
        & (\%) $\uparrow$
        & (\%) $\uparrow$
        & (\%) $\uparrow$
        & (\%) $\uparrow$
        & (eV/atom) $\downarrow$
        & (\AA) $\downarrow$ \\
        \midrule

        \multirow{4}{*}{MatterGen}
            & 1
            & 1.116
            & 49.32
            & 99.02
            & 82.42
            & 33.59
            & 0.1847
            & 0.1961 \\

            & 3
            & 1.724
            & 45.90
            & 98.93
            & 83.98
            & 31.05
            & 0.1530
            & 0.2684 \\

            & 5
            & 2.185
            & 49.51
            & 94.82
            & 82.32
            & 31.54
            & 0.1513
            & 0.3016 \\

            & 7
            & 2.622
            & 55.37
            & 86.82
            & 84.18
            & 35.16
            & 0.1456
            & 0.3338 \\

        \midrule

        \multirow{4}{*}{\textbf{OxiGen}}
            & 1
            & 1.138
            & 60.55
            & 96.48
            & 73.63
            & 35.84
            & 0.1489
            & 0.1813 \\

            & 3
            & 1.322
            & 71.48
            & 96.29
            & 71.97
            & 43.55
            & 0.0877
            & 0.1895 \\

            & 5
            & 1.486
            & 72.75
            & 91.11
            & 70.41
            & 42.97
            & 0.0816
            & 0.2303 \\

            & 7
            & 2.125
            & 68.95
            & 87.79
            & 75.98
            & 41.89
            & 0.0878
            & 0.2580 \\

        \bottomrule
    \end{tabular*}
\end{table}

\newpage

\subsection{Example Crystals}

\begin{figure}[h]
    \centering
    \setlength{\tabcolsep}{2pt}

    \begin{tabular}{@{}ccc@{}}

        \begin{tabular}{@{}c@{}}
            \includegraphics[
                width=0.31\textwidth
            ]{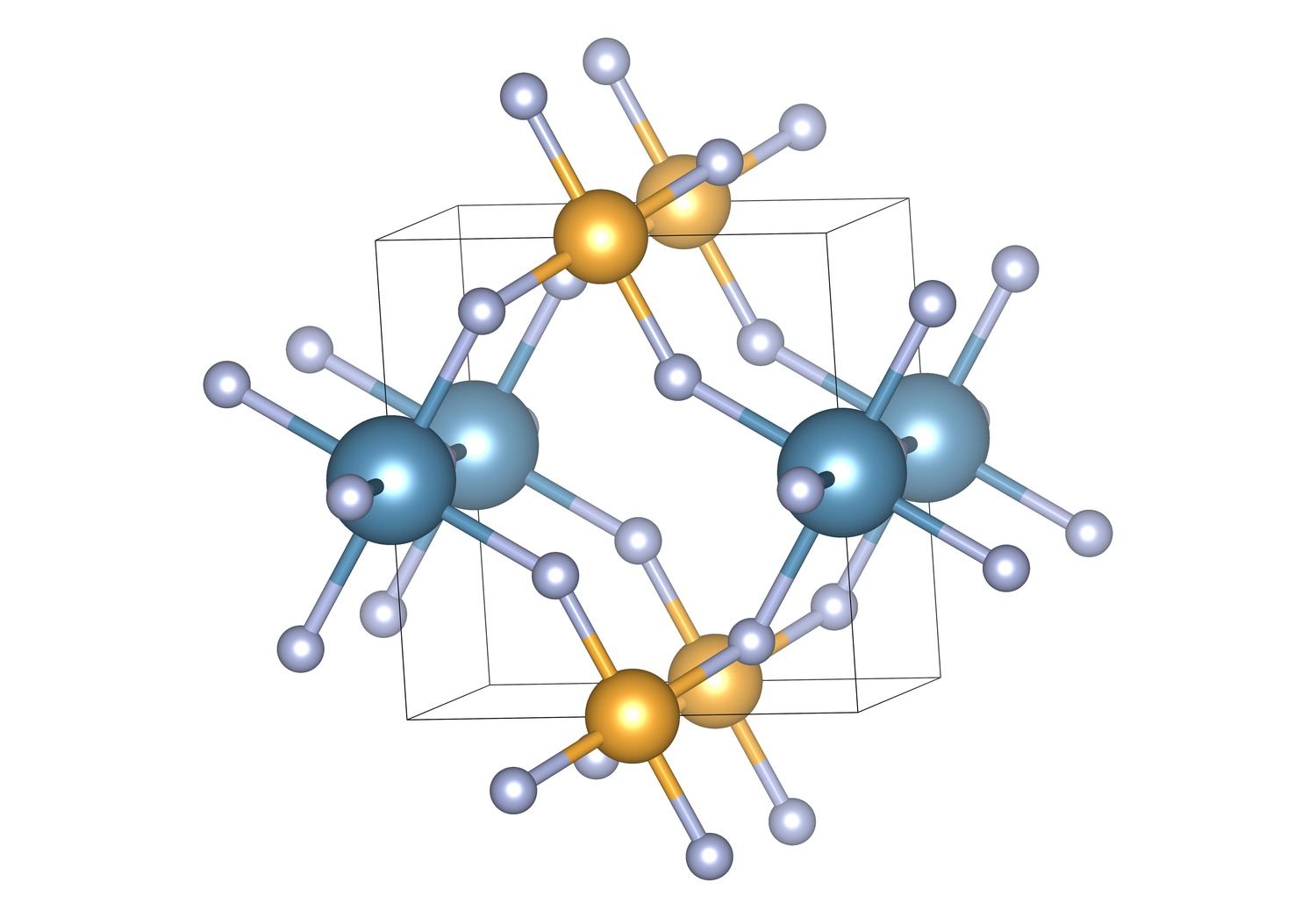} \\
            \small \makebox[0pt][r]{\textbf{(a)}\enspace}CaAuF$_5$ \\[-1pt]
            \scriptsize Ca$^{2+}$ Au$^{3+}$ F$^{-}$
        \end{tabular}
        &
        \begin{tabular}{@{}c@{}}
            \includegraphics[
                width=0.31\textwidth
            ]{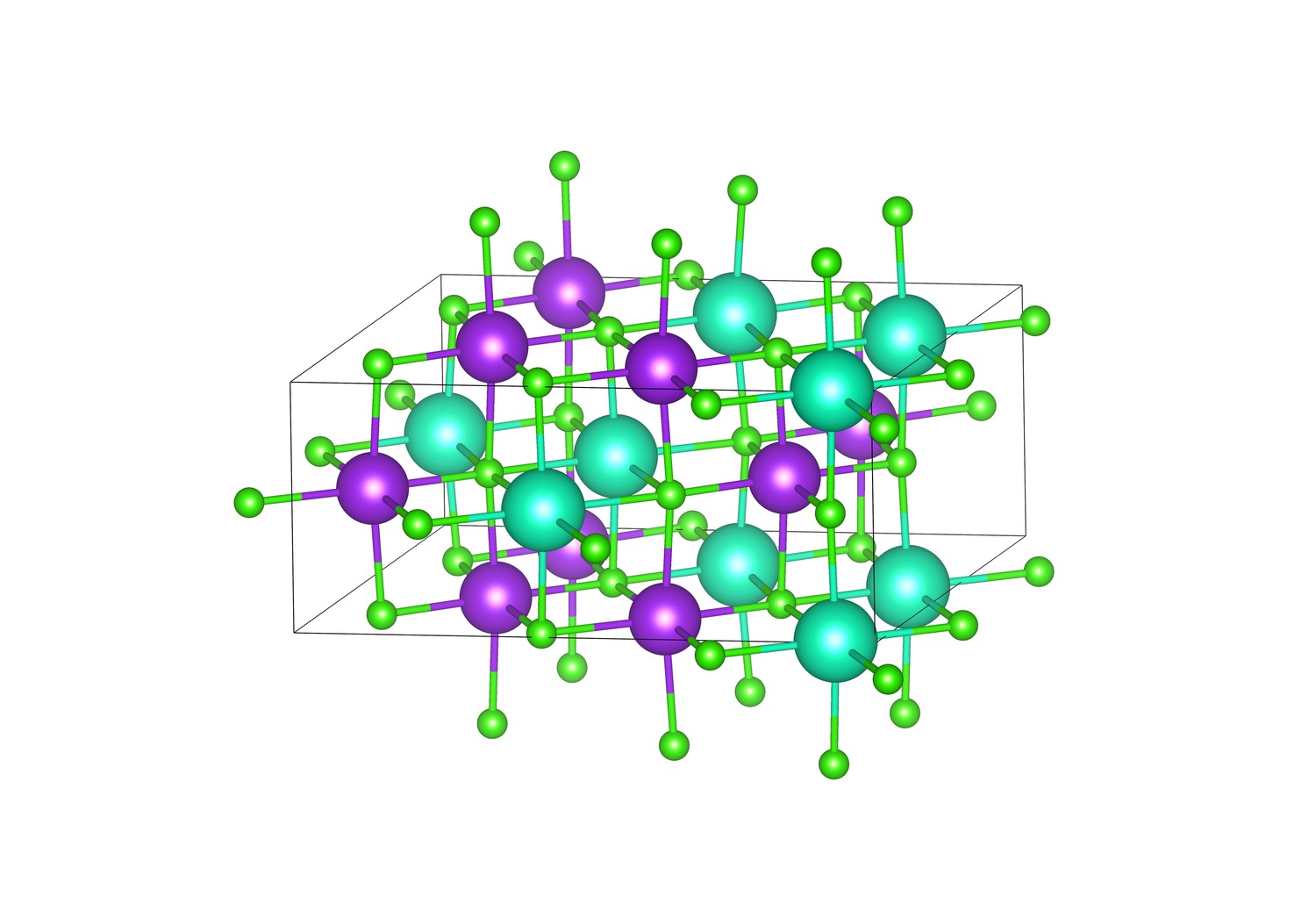} \\
            \small \makebox[0pt][r]{\textbf{(b)}\enspace}CsKCl$_2$ \\[-1pt]
            \scriptsize Cs$^{+}$ K$^{+}$ Cl$^{-}$
        \end{tabular}
        &
        \begin{tabular}{@{}c@{}}
            \includegraphics[
                width=0.31\textwidth
            ]{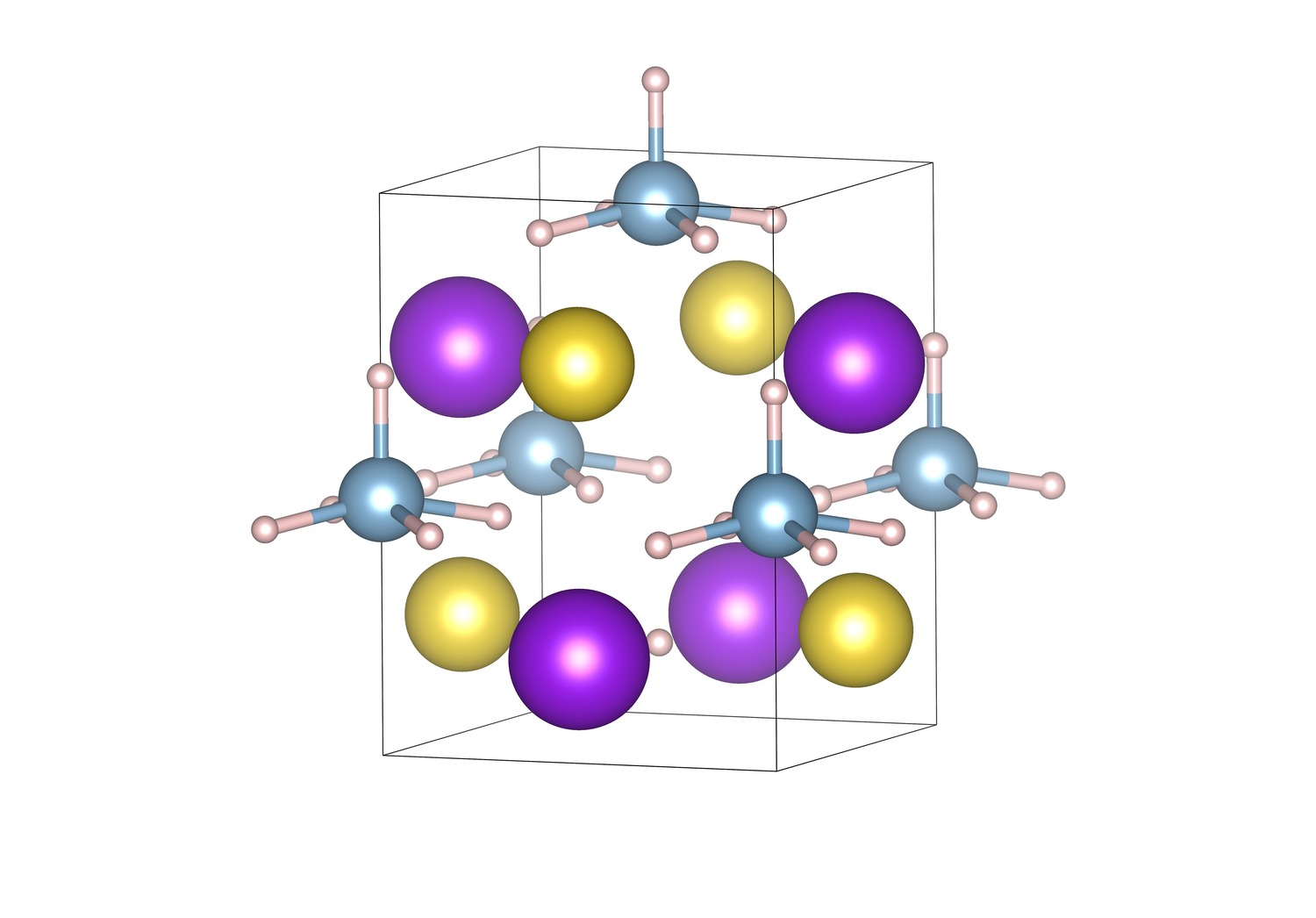} \\
            \small \makebox[0pt][r]{\textbf{(c)}\enspace}KNaAlH$_5$ \\[-1pt]
            \scriptsize K$^{+}$ Na$^{+}$ Al$^{3+}$ H$^{-}$
        \end{tabular}

        \\[2mm]

        \begin{tabular}{@{}c@{}}
            \includegraphics[
                width=0.31\textwidth
            ]{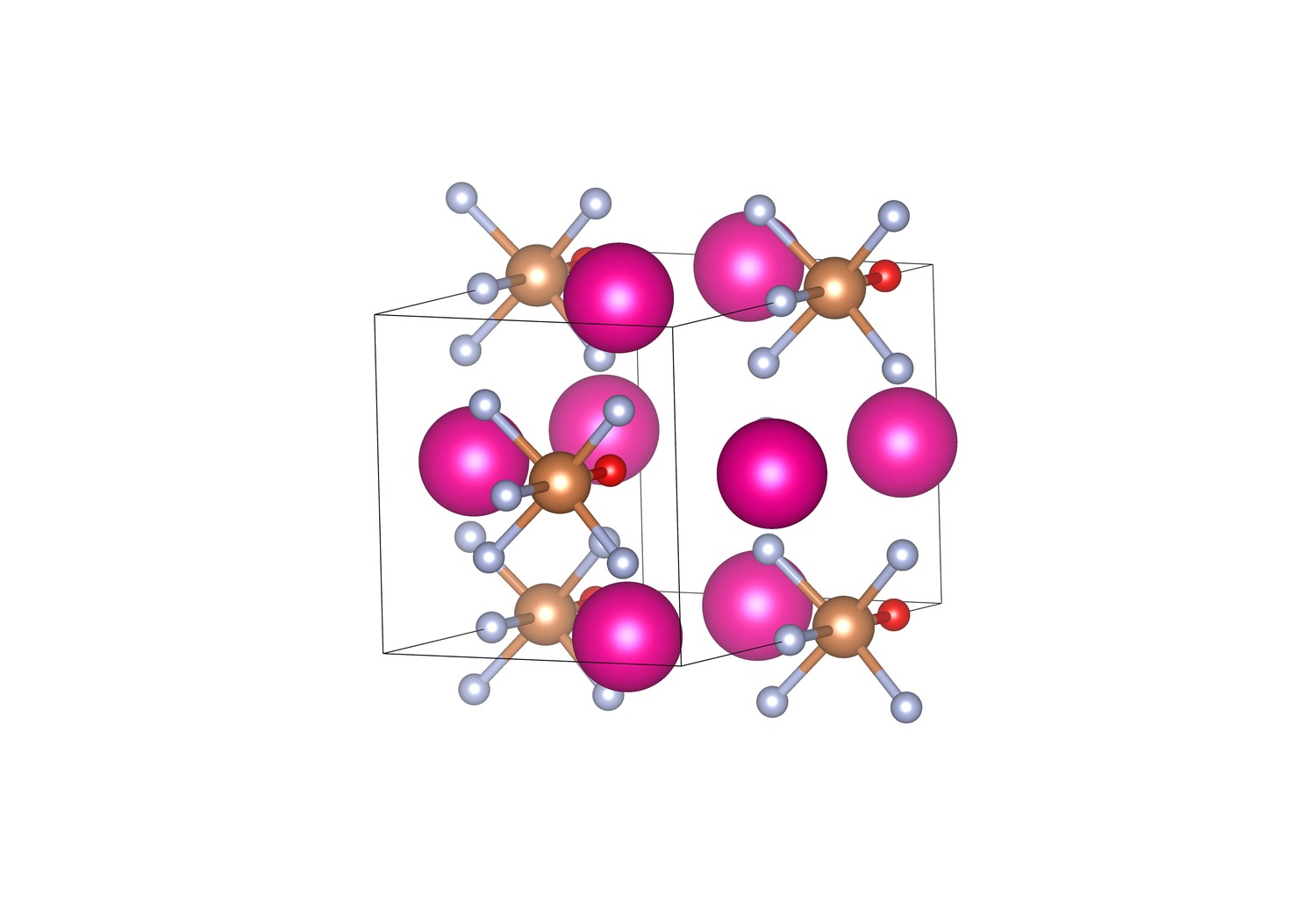} \\
            \small \makebox[0pt][r]{\textbf{(d)}\enspace}Rb$_2$SbOF$_5$ \\[-1pt]
            \scriptsize Rb$^{+}$ Sb$^{5+}$ O$^{2-}$ F$^{-}$
        \end{tabular}
        &
        \begin{tabular}{@{}c@{}}
            \includegraphics[
                width=0.31\textwidth
            ]{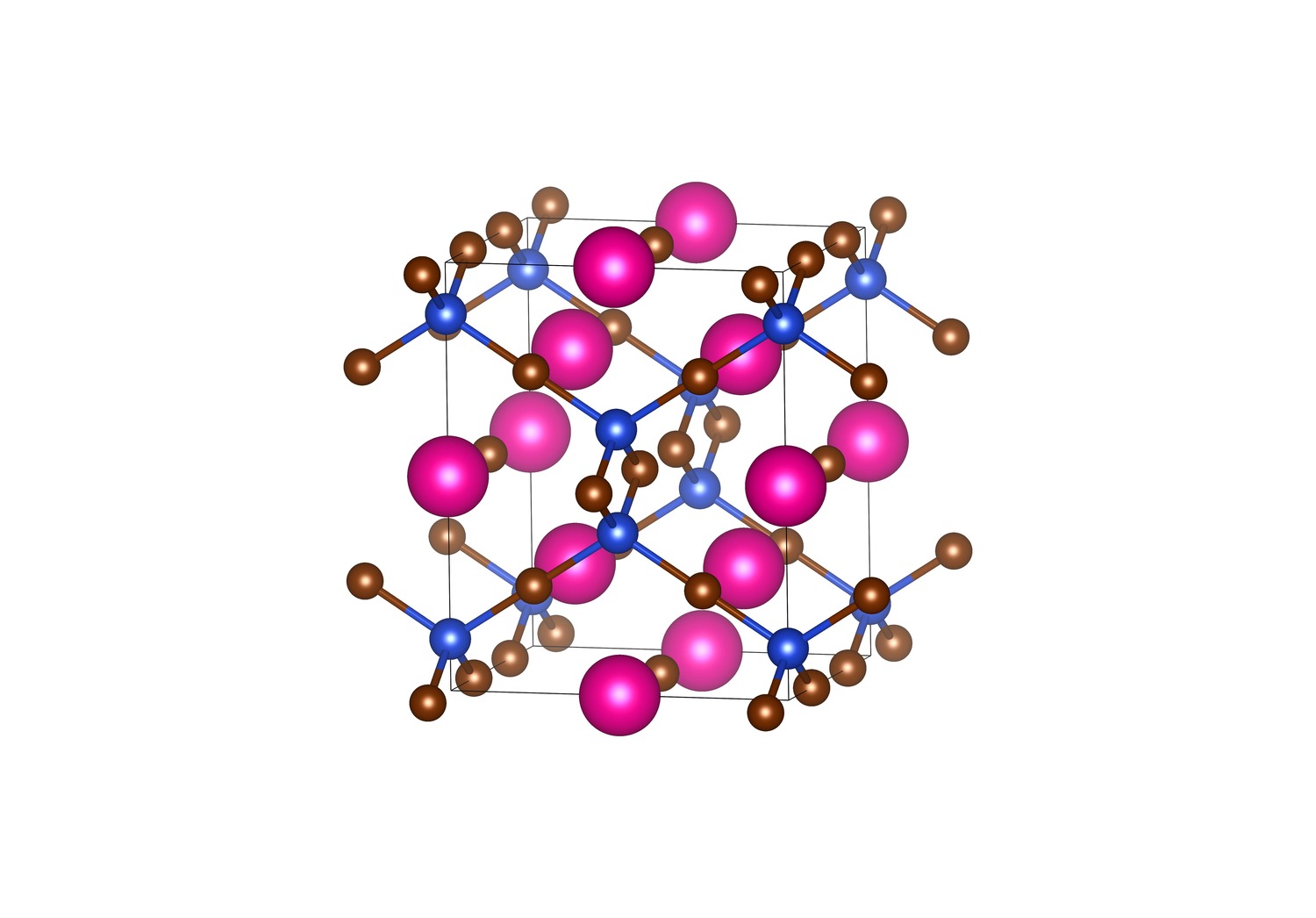} \\
            \small \makebox[0pt][r]{\textbf{(e)}\enspace}Rb$_3$Cu$_2$Br$_5$ \\[-1pt]
            \scriptsize Rb$^{+}$ Cu$^{+}$ Br$^{-}$
        \end{tabular}
        &
        \begin{tabular}{@{}c@{}}
            \includegraphics[
                width=0.31\textwidth
            ]{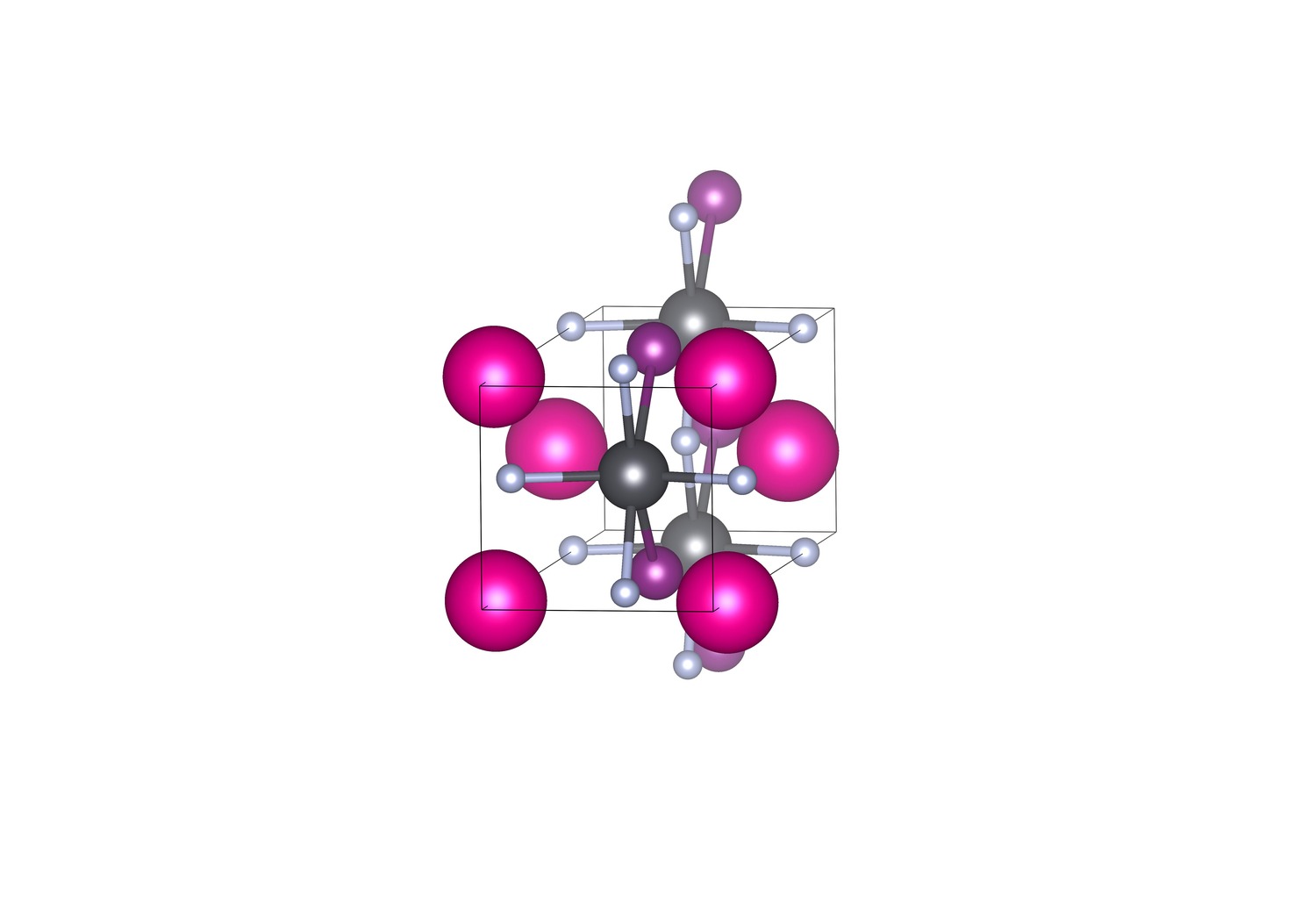} \\
            \small \makebox[0pt][r]{\textbf{(f)}\enspace}RbPbIF$_2$ \\[-1pt]
            \scriptsize Rb$^{+}$ Pb$^{2+}$ I$^{-}$ F$^{-}$
        \end{tabular}

        \\[2mm]

        \begin{tabular}{@{}c@{}}
            \includegraphics[
                width=0.31\textwidth
            ]{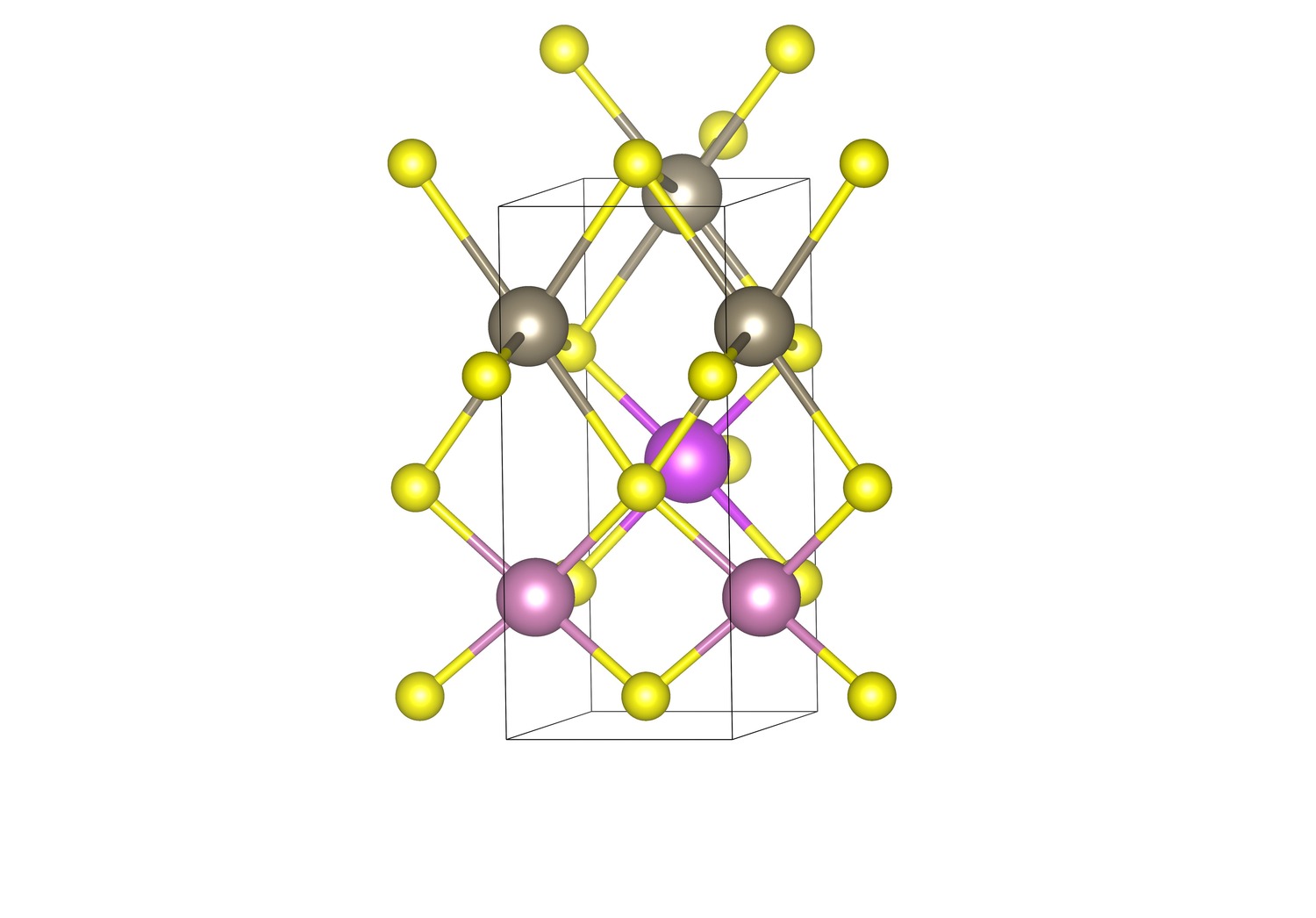} \\
            \small \makebox[0pt][r]{\textbf{(g)}\enspace}Tl$_2$InBiS$_4$ \\[-1pt]
            \scriptsize Tl$^{+}$ In$^{3+}$ Bi$^{3+}$ S$^{2-}$
        \end{tabular}
        &
        \begin{tabular}{@{}c@{}}
            \includegraphics[
                width=0.31\textwidth
            ]{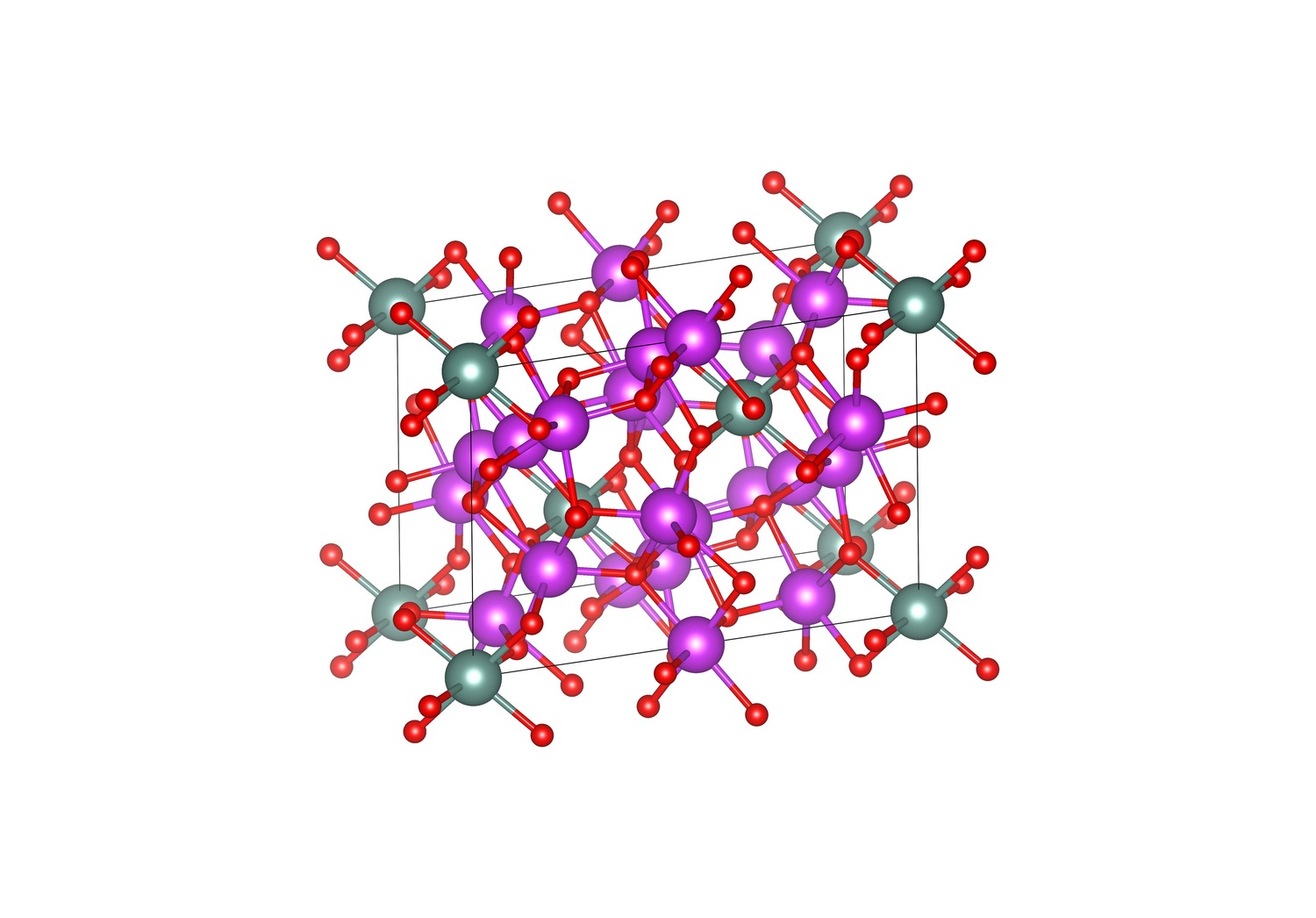} \\
            \small \makebox[0pt][r]{\textbf{(h)}\enspace}YBi$_7$O$_{12}$ \\[-1pt]
            \scriptsize Y$^{3+}$ Bi$^{3+}$ O$^{2-}$
        \end{tabular}
        &
        \begin{tabular}{@{}c@{}}
            \includegraphics[
                width=0.31\textwidth
            ]{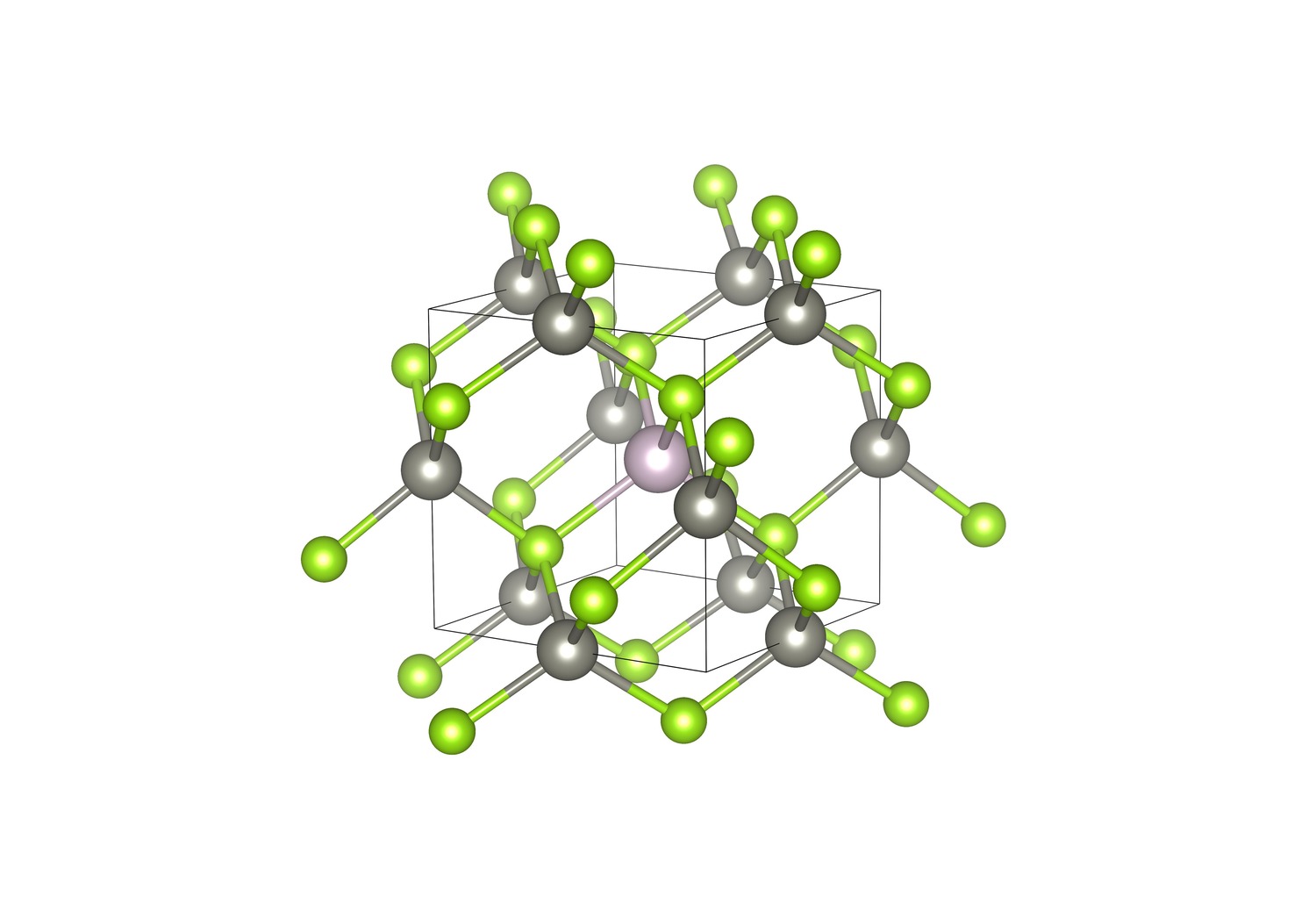} \\
            \small \makebox[0pt][r]{\textbf{(i)}\enspace}Zn$_3$HgSe$_4$ \\[-1pt]
            \scriptsize Zn$^{2+}$ Hg$^{2+}$ Se$^{2-}$
        \end{tabular}

    \end{tabular}

    \caption{Examples of stable, unique, and novel crystals generated by OxiGen.}
    \label{fig:crystal-grid}
\end{figure}

\end{document}